\documentclass{article}

\usepackage{microtype}
\usepackage{graphicx}
\usepackage{subfigure}
\usepackage{booktabs} 
\usepackage{mathrsfs}

\usepackage{thm-restate}

\usepackage{nicefrac}
\usepackage{xcolor}
\usepackage{amsmath,amssymb,amsthm}

\theoremstyle{definition}

\usepackage{enumitem}
\setlist[enumerate]{leftmargin=0.5cm,topsep=0pt,itemsep=-2pt}
\setlist[itemize]{leftmargin=0.5cm,topsep=0pt,itemsep=-2pt}

\usepackage{mathrsfs}
\newcommand{\probspace}{\mathscr{P}}

\usepackage{mathtools}
\usepackage{hyperref}
\hypersetup{
    colorlinks=true,
    linkcolor=blue,
    citecolor=cyan,
}

\usepackage[accepted]{icml2021}

\icmltitlerunning{VA-learning as a more efficient alternative to Q-learning}

\begin{document}

\twocolumn[
\icmltitle{VA-learning as a more efficient alternative to Q-learning}



\icmlsetsymbol{equal}{*}

\begin{icmlauthorlist}
\icmlauthor{Yunhao Tang}{dm}
\icmlauthor{R\'emi Munos}{dm}
\icmlauthor{Mark Rowland}{dm}
\icmlauthor{Michal Valko}{dm}
\end{icmlauthorlist}

\icmlaffiliation{dm}{Google DeepMind}

\icmlcorrespondingauthor{Yunhao Tang}{robintyh@deepmind.com}

\icmlkeywords{Machine Learning, ICML}

\vskip 0.3in
]



\printAffiliationsAndNotice{\icmlEqualContribution} 

\begin{abstract}
In reinforcement learning, the advantage function is critical for policy improvement, but is often extracted from a learned Q-function. A natural question is: \textit{Why not learn the advantage function directly?} In this work, we introduce VA-learning, which directly learns advantage function and value function using bootstrapping, \textit{without} explicit reference to Q-functions. VA-learning learns off-policy and enjoys similar theoretical guarantees as Q-learning. Thanks to the direct learning of advantage function and value function, VA-learning improves the sample efficiency over Q-learning both in tabular implementations and deep RL agents on Atari-57 games. We also identify a close connection between VA-learning and the dueling architecture, which partially explains why a simple architectural change to DQN agents tends to improve performance.
\end{abstract}

\section{Introduction}

Developed just over three decades ago, Q-learning \citep{watkins1989learning,watkins1992q} is one of the most fundamental algorithms of reinforcement learning (RL). Q-learning progresses in an iterative fashion, updating the current value predictions by bootstrapping from its own future value predictions. In addition to its theoretical appeal, the incremental nature of Q-learning is also compatible with powerful deep learning machinery, which has fueled recent breakthroughs in Atari games
\citep{mnih2013}.

Q-learning learns the Q-function $Q(x,a)$, defined as the expected return obtained starting from certain state-action pair, and executing an optimal policy. By splitting the Q-function into a state-dependent value function $V(x)$ and a residual advantage function $A(x,a)$, we arrive at the commonly used decomposition
\begin{align*}
    Q(x,a) = V(x) + A(x,a).
\end{align*}
In many situations accurately approximating the advantage function, which measures the relative performance between actions, is the end goal of the algorithm. Instead of learning advantage functions implicitly via Q-functions, a natural question is whether it is possible to learn advantage functions directly. Unfortunately, unlike Q-functions, the advantage function does not obey a recursive equation (like the Bellman equation for Q-functions) and cannot be learned as a standalone object by bootstrapping from itself.

Our key remedy to resolving the above dilemma is \textit{learning an extra value function at the same time}. We introduce VA-learning (Section~\ref{sec:VA-learning}), an algorithm that derives it name from the fact that it directly learns a value function $V$ and an advantage function $A$. Importantly, the decomposition $Q=V+A$ does not constrain us from learning just the target value functions. In fact, as we will explain in detail, VA-learning derives its properties by learning a value function adapted to the data collection policy. On a high level, VA-learning is reminiscent of the dueling architecture for Q-learning \citep{wang2016}, which runs a vanilla Q-learning algorithm with a parameterization that decomposes Q-functions into value and advantage functions. While the dueling architecture is purely empirically motivated, we provide grounded theoretical guarantees to the performance of VA-learning.

Besides theoretical guarantees, we also find that in practice VA-learning is generally more superior to vanilla Q-learning, in both tabular and deep RL settings. A high level explanation is that through the decomposition $Q=V+A$, VA-learning explicitly allows for an extra degree of freedom such that the learning takes place at different rate across different components of the Q-function. Concretely, we generally expect the $V$ to be learned more quickly than $A$, as the former is shared across all actions. In the case of bootstrapped updates, this technique helps to increase the speed at which the advantage function and target Q-function are learned.
We now highlight a few crucial detailed properties enjoyed by VA-learning. 

\paragraph{Theoretical guarantee as Q-learning.} VA-learning enjoys the same theoretical guarantee as Q-learning (Section~\ref{sec:VA-learning}). The \emph{implied} Q-function of VA-learning $Q=V+A$ converges to the same target fixed point as Q-learning, whereas $V$ and $A$ converge to properly defined value and advantage function respectively.

\paragraph{Improved efficiency of tabular algorithms.} 
Through the decomposition $Q=V+A$, VA-learning effectively allows the shared part of the Q-function to be learned \emph{quickly} via the~$V$ component, and more \emph{slowly} via the~$A$ component.
When the learning targets are computed using bootstrapping, the accelerated effect of such a decomposition becomes even more profound, even in simple tabular MDPs (Section~\ref{sec:VA-learning})

\paragraph{Large-scale value-based learning.} VA-learning can also be used as a component within large-scale RL agents (Section~\ref{sec:b-dueling}). When implemented with function approximation, we draw an intriguing connection between VA-learning and the dueling architecture \citep{wang2016}. On the Atari-57 game suite, VA-learning provides robust improvements over the dueling and Q-learning baselines (Section~\ref{sec:exp}).

\section{Background}
Consider a Markov decision process (MDP) represented as the tuple $\left(\mathcal{X},\mathcal{A},P_R,P,\gamma\right)$ where $\mathcal{X}$ is a finite state space, $\mathcal{A}$ the finite action space, $P_R:\mathcal{X}\times\mathcal{A}\rightarrow \probspace(\mathbb{R})$ the reward kernel, $P:\mathcal{X}\times\mathcal{A}\rightarrow\probspace(\mathcal{X})$ the transition kernel and $\gamma\in [0,1)$ the discount factor. For any policy $\pi:\mathcal{X}\rightarrow\probspace(\mathcal{A})$, important quantities include Q-function $Q^\pi(x,a)\coloneqq\mathbb{E}_\pi\left[\sum_{t=0}^\infty\gamma^tr_t\;\middle|\;x_0=x,A_0=a\right]$, value function $V^\pi(x)\coloneqq \sum_a \pi(a|x) Q^\pi(x,a) $ and advantage function $A^\pi(x,a)\coloneqq Q^\pi(x,a)-V^\pi(x)$.

In policy evaluation, the aim is to compute the target Q-function $Q^\pi$ for a fixed target policy $\pi$. The target Q-function $Q^\pi$ can be approximated by applying the recursion
$
     Q_{t+1}= \mathcal{T}^\pi Q_t,
$
where $\mathcal{T}^\pi:\mathbb{R}^{\mathcal{X}\times\mathcal{A}}\rightarrow \mathbb{R}^{\mathcal{X}\times\mathcal{A}}$ is the Bellman evaluation operator. In control, the aim is to find an optimal policy $\pi^\star(\cdot|x)\coloneqq \arg\max_a Q^\ast(x,a)$ 
with Q-function $Q^\star(x,a)\coloneqq \max_\pi Q^\pi(x,a)$. It can be can be approximated, by applying the recursion
$
    Q_{t+1}= \mathcal{T}^\star Q_t$, 
with the Bellman control operator $\mathcal{T}^\star$.

In most applications, it is   infeasible to compute the above recursions exactly as they require analytic knowledge of the transition and reward kernel. Instead, from a given state $x\in\mathcal{X}$, it is more common to access a sampled transition $(x_t,a_t,r_t,x_{t+1})$ tuple at step $t\geq 0$,
\begin{align*}
    a_t\sim \mu(\cdot|x_t), r_t\sim P_R(\cdot|x_t,a_t), x_{t+1}\sim P(\cdot|x_t,a_t),
\end{align*}
where $\mu$ is the behavior policy, which for simplicity is assumed fixed and has full coverage over the entire action space $\mu(a|x)>0,\forall (x,a)\in\mathcal{X}\times\mathcal{A}$. Let $(p_t)_{t=0}^\infty$ and $(q_t)_{t=0}^\infty$ be any number arrays, in the following, we also use $p_{t+1}\!\overset{\alpha_t}{\leftarrow}q_t$ as shorthand notation for the incremental update $p_{t+1}=p_t+\alpha_t(q_t - p_t)$ with learning rate $\alpha_t$.

Let $(Q_t)_{t=0}^\infty$ be a sequence of estimated Q-functions. First we consider the update for the policy evaluation case, which is commonly known as TD-learning,
\begin{align}
    Q_{t+1}(x_t,a_t)\overset{\alpha_t}{\leftarrow} r_t + \gamma Q_t (x_{t+1},\pi),\label{eq:q-learning-eval}
\end{align}
where $Q_t(x,\pi) \coloneqq \sum_a \pi(a|x) Q_t(x,a)$. 
The back-up target $r_t + \gamma Q_t (x_{t+1},\pi)$ can be understood as a stochastic approximation to the evaluation Bellman recursion back-up target $\mathcal{T}^\pi Q_t$. 
In the control case, the Q-learning update is 
\begin{align}
   Q_{t+1}(x_t,a_t)\overset{\alpha_t}{\leftarrow} r_t + \gamma \max_a Q_t (x_{t+1},a).\label{eq:q-learning-control}
\end{align}
With a properly chosen learning rate scheme $(\alpha_t)_{t=0}^\infty$ and mild assumptions on the data process \citep{watkins1992q,tsitsiklis1994asynchronous,jaakkola1994onvergence}, TD-learning and Q-learning converge almost surely to $Q^\pi$ or $Q^\star$ respectively. 

\section{VA-learning}
\label{sec:VA-learning}

We now introduce VA-learning, the central object of stufy of the paper. At iteration $t$, VA-learning maintains a value function estimate $V_t(x)$ and advantage function estimate $A_t(x,a)$. Most importantly, unlike Q-learning, VA-learning does \emph{not} maintain a separate Q-function. To recover a Q-function estimate, VA-learning combines the value and advantage function estimate as 
\begin{align*}
    Q_t(x,a)\coloneqq V_t(x)+A_t(x,a),\forall (x,a)\in\mathcal{X}\times\mathcal{A}.
\end{align*}

\begin{algorithm}[t]
\label{algo:bdueling}
\begin{algorithmic}
\STATE  Initializations $V_0\in\mathbb{R}^{\mathcal{X}}$ and $A_0\in\mathbb{R}^{\mathcal{X}\times\mathcal{A}}$; behavior policy $\mu$ and learning rate sequence $(\alpha_t)_{t=0}^\infty$.
\FOR{$t=0,1,2,\dots,K$}
\STATE \textbf{Step 1.} Sample transition $(x_t,a_t,r_t,x_{t+1})$.
\STATE \textbf{Step 2.} Let $Q_t(x_t,a_t)=V_t(x_t)+A_t(x_t,a_t)$. Compute back-up target $\widehat{\mathcal{T}}Q_t(x_t,a_t)$ based on Eqn~\eqref{eq:vabackup-eval} for policy evaluation and Eqn~\eqref{eq:vabackup-control} for control.
\STATE \textbf{Step 3.} Update the value and advantage iterates
\begin{align*}
    V_{t+1}(x_t) &\overset{\alpha_t}{\leftarrow} \widehat{\mathcal{T}} Q_t(x_t,a_t) - \gamma A_t(x_{t+1},\mu), \nonumber \\
    A_{t+1}(x_t,a_t) &\overset{\alpha_t}{\leftarrow} \widehat{\mathcal{T}} Q_t(x_t,a_t) - \gamma A_t(x_{t+1},\mu) - V_t(x_t).
\end{align*}
\ENDFOR
\STATE  Output final $V_t$ and $a_t$.
\caption{Tabular VA-learning}
\end{algorithmic}
\end{algorithm}

\subsection{Policy evaluation and control}

Throughout, we assume access to the transition tuple $(x_t,a_t,r_t,x_{t+1})$ at time $t$ as in the TD-learning and Q-learning case. 
To better highlight the difference between VA-learning and TD-learning (the difference is similar between VA-learning and Q-learning), we start by defining the policy evaluation back-up target for TD-learning,
\begin{align}
    \widehat{\mathcal{T}}^\pi Q_t(x_t,a_t)\coloneqq r_t + \gamma Q_t(x_{t+1},\pi).\label{eq:vabackup-eval}
\end{align}
The policy evaluation recursion in Eqn~\eqref{eq:q-learning-eval} rewrites as 
$
    Q_{t+1}(x_t,a_t) \overset{\alpha_t}{\leftarrow} \widehat{\mathcal{T}}^\pi Q_t(x_t,a_t)
$. In contrast, policy evaluation VA-learning carries out the following recursion:
\begin{align}\tag{Policy evaluation VA-learning}
\begin{split}
    V_{t+1}(x_t) &\overset{\alpha_t}{\leftarrow} \widehat{\mathcal{T}}^\pi Q_t(x_t,a_t) - \gamma A_t(x_{t+1},\mu), \nonumber \\
    A_{t+1}(x_t,a_t) &\overset{\alpha_t}{\leftarrow} \widehat{\mathcal{T}}^\pi Q_t(x_t,a_t) - \gamma A_t(x_{t+1},\mu) - V_t(x_t). 
\end{split}
\end{align}
where we similarly define $A_t(x,\mu) \coloneqq \sum_a \mu(a|x) A_t(x,a)$.

\paragraph{Understanding the back-up targets.} To better understand the updates, note that the back-up targets for value estimate $V_t$ and advantage estimate $A_t$ share the common back-up target $\widehat{\mathcal{T}}^\pi Q_t(x_t,a_t)-A_t(x_{t+1},\mu)$. To better understand the back-up target, we rewrite it as the estimated Bellman operator $\widehat{\mathcal{T}}^\pi$ applied to a transformed Q-function $\widetilde{Q}_t(x_t,a_t)$,
\begin{align*}
    &\widehat{\mathcal{T}}^\pi Q_t(x_t,a_t) - \gamma A_t(x_{t+1},\mu) = \widehat{\mathcal{T}}^\pi \widetilde{Q}_t(x_t,a_t), 
\end{align*}
Here, the transformed Q-function $\widetilde{Q}_t(x_t,a_t)=V_t(x_t) + \widetilde{A}_t(x_t,a_t)$ has a special parameterization of its advantage function 
\begin{align*}
    \widetilde{A}_t(x_t,a_t) = A_t(x_t,a_t)-A_t(x_t,\mu)
\end{align*}
such that the advantage function has zero mean $\widetilde{A}_t(x_t,\mu)=0$ under behavior policy $\mu$.
Intriguingly, such a transformation is reminiscent of though distinct from
the dueling architecture for DQN \citep{wang2016}. We will draw further connections between VA-learning and dueling in Section~\ref{sec:b-dueling}.

We can interpret the value function estimate $V_{t+1}(x_t)$ as learning the \emph{average} of the common back-up targets, averaged over all actions taken from state $x$. Meanwhile, the advantage function estimate $A_{t+1}(x_t,a_t)$ learns the \emph{residual} of the back-up target (after subtracting the baseline $V_t(x)$). Intuitively, this hints at the fact that $V_t,A_t$ 
indeed learn certain value functions and advantage functions respectively. We will make the convergence behavior and fixed points of VA-learning more clear shortly. 

\paragraph{Control case.}
For the control case, we define the control back-up target similar to Q-learning,
\begin{align}
    \widehat{\mathcal{T}}^\star Q_t(x_t,a_t)\coloneqq r_t + \gamma \max_a Q_t(x_{t+1},a).\label{eq:vabackup-control}
\end{align}
Then control VA-learning carries out the recursion:
\begin{align}\tag{Control VA-learning}
\begin{split}
    V_{t+1}(x_t) &\overset{\alpha_t}{\leftarrow} \widehat{\mathcal{T}}^\star Q_t(x_t,a_t) - \gamma A_t(x_{t+1},\mu), \nonumber \\
    A_{t+1}(x_t,a_t) &\overset{\alpha_t}{\leftarrow} \widehat{\mathcal{T}}^\star Q_t(x_t,a_t) - \gamma A_t(x_{t+1},\mu) - V_t(x_t).
\end{split}
\end{align}

\subsection{Why VA-learning can be more efficient} \label{sec:VA-learning-why}

\begin{figure}[t]
    \centering
    \includegraphics[keepaspectratio,width=.32\textwidth]{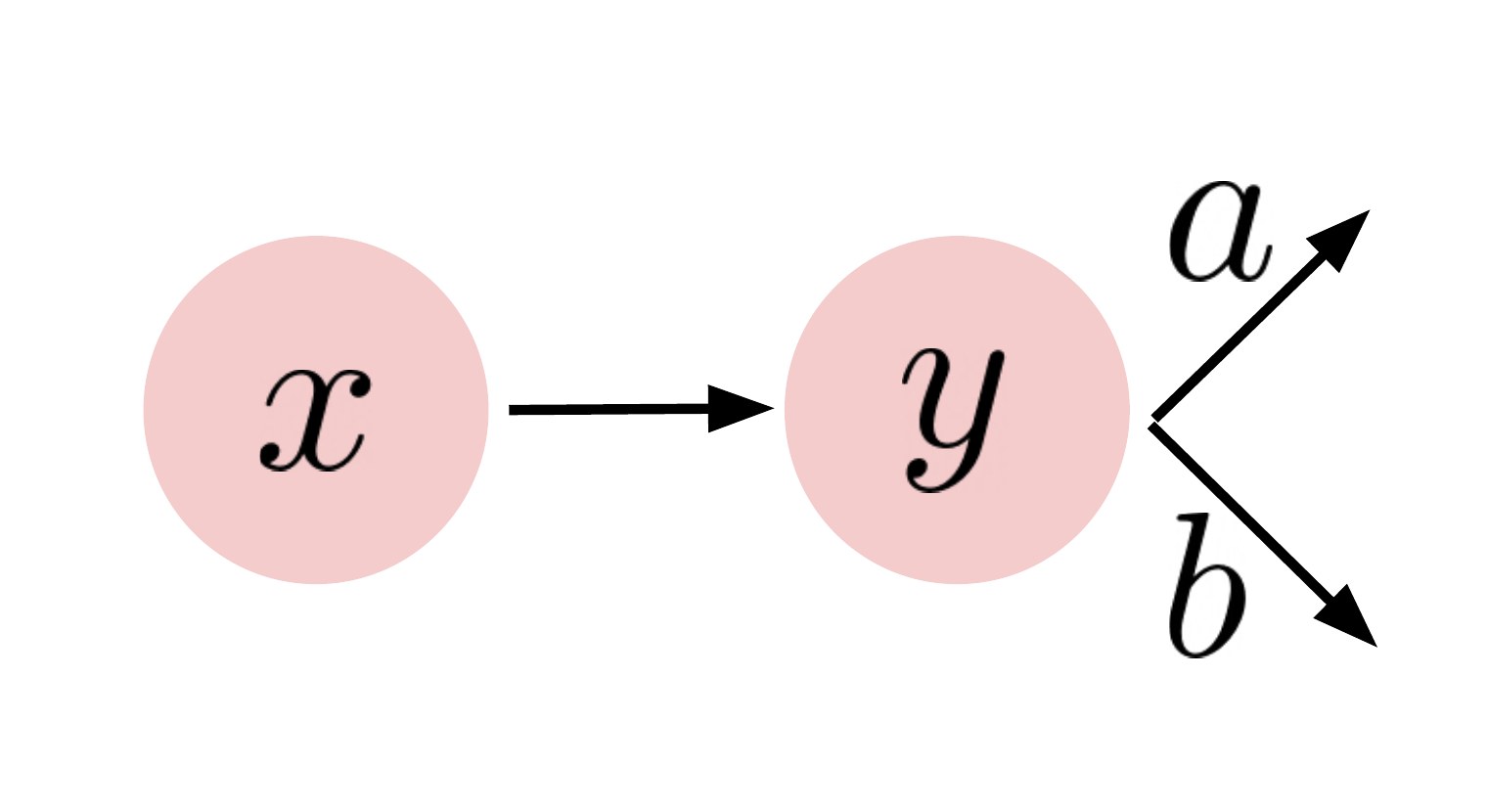}
    \caption{A simple scenario to illustrate the effectiveness of VA-learning over Q-learning. There are two states $x,y$ and from state $y$ there are two actions $a,b$. Assume there is a back-up target $\widehat{Q}(y,a)$, VA-learning will update the prediction for both $Q(y,a)$ and $Q(y,b)$ thanks to the shared value function $V(y)$. In contrary, Q-learning only updates the prediction $Q(y,a)$. The accelerated learning of $Q(y,b)$ helps accelerate learning $Q(x,\cdot)$ when bootstrapping from $Q(y,\cdot)$.}
    \label{fig:tabular-mdp-example}
\end{figure}

Before formally presenting the convergence behavior of VA-learning, we provide intuitive explanations and numerical examples to show why VA-learning can be often more efficient than TD-learning and Q-learning.

\paragraph{An illustrative example.} Consider a fixed state $y$ with two actions $a,b$. Imagine so far only action $a$ has been sampled from state $y$, TD-learning or Q-learning would have updated Q-function estimate $Q(y,a)$, while the Q-function estimate $Q(y,b)$ has never been updated. Now, for any state $x$ that precedes state~$y$, constructing the Q-learning back-up target at state $x$ may require bootstrapping from $Q(y,b)$. Since $Q(y,b)$ is never updated before, such a back-up target for state $x$ is of low quality. Nevertheless, Q-learning can still work by generating more data until the action $b$ at state $y$ is sampled more time. However, the situation above implies that propagating the correct information from $y$ to $x$ can be slowed down by not having enough transitions $(y,b)$ sampled.

For VA-learning, when the action $a$ is sampled at state $y$, the back-up target for $(y,a)$ will be split into back-up targets for $V(y)$ and $A(y,a)$. This ensures both $V(y)$ and $A(y,a)$ are updated to certain extent. Now, at the preceding state~$x$, when bootstrapping from $(y,b)$ to construct its back-up target, we effectively bootstrap from
$
    Q(y,b) = V(y) + A(y,b)
$. Although $A(y,b)$ has not been updated before, the bootstrap target can still utilize information contained in the updated value function estimate $V(y)$. This means the VA-learning back-up target at state $x$ is already potentially more informative compared to its counterpart in Q-learning. 

In summary, the potential benefits of VA-learning come from the decomposition of Q-function into a value function and an advantage function. Since the value function is shared across all actions, it can be learned faster by pooling back-up targets across all actions. When used as bootstrapped targets, the induced Q-function benefits from information contained in the value function, which in turn accelerates the learning process. 

To empirically validate the above claims, we examine the performance of Q-learning and VA-learning under the policy evaluation case in tabular MDPs. We carry out recursive updates based on a fixed number of trajectories under behavior policy $\mu$. We examine the error of the advantage estimate $\left\lVert \widehat{A}_t - A^\pi\right\rVert_2$ at update iteration $k$, as the advantage function error is also indicative of performance in the control case. VA-learning provides significant improvements over Q-learning both in terms of convergence speed and asymptotic accuracy. Detailed results are shown in Figure~\ref{fig:tabular-adv} (Appendix~\ref{appendix:exp}).

\begin{figure}[t]
    \centering
    \includegraphics[keepaspectratio,width=.42\textwidth]{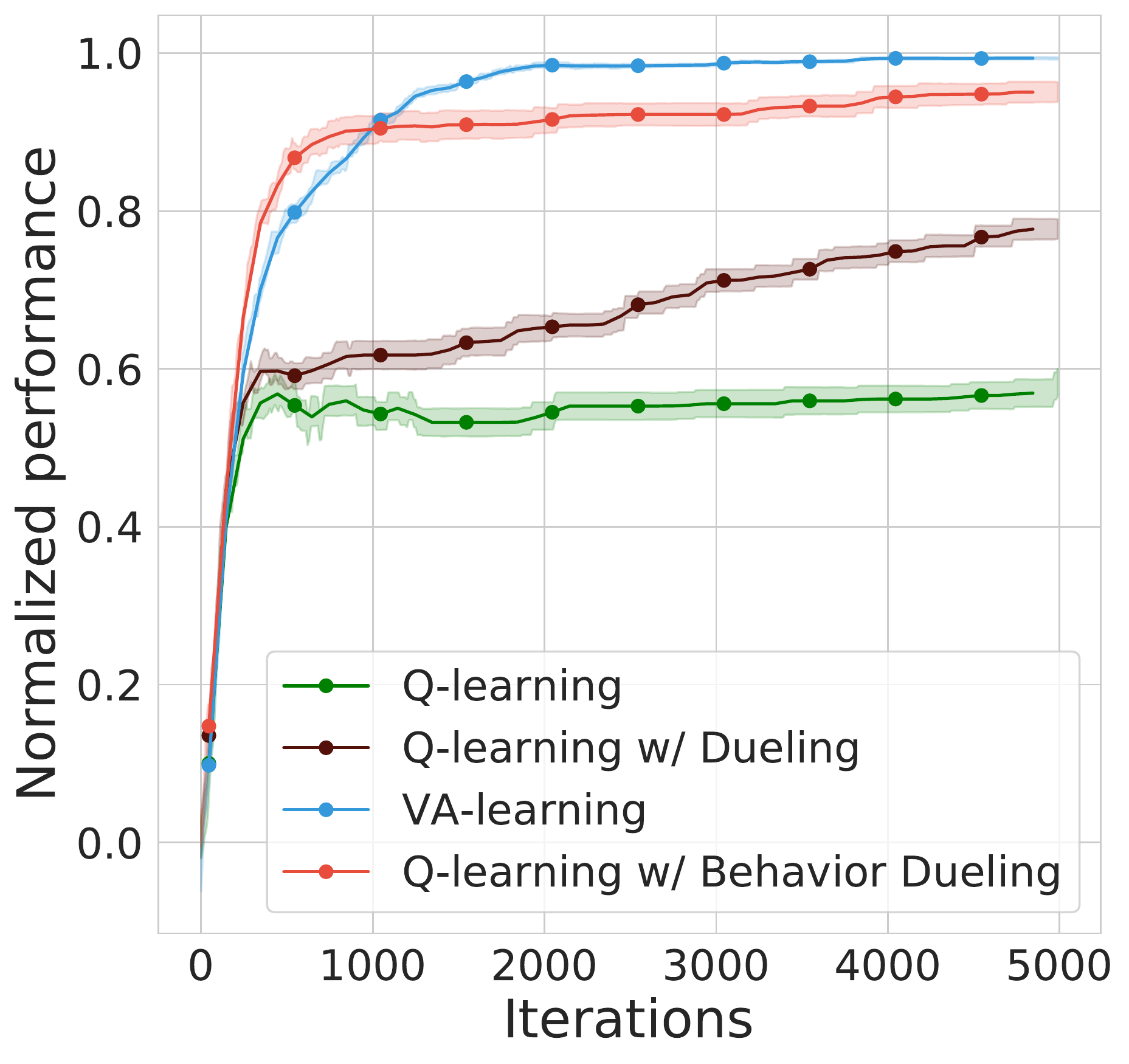}
    \caption{Comparing VA-learning (Section~\ref{sec:VA-learning}), Q-learning with behavior dueling (Section~\ref{sec:b-dueling}), Q-learning with uniform dueling \citep{wang2016sample} and regular Q-learning. We experiment on tabular MDPs with fixed behavior policy $\mu=\epsilon u + (1-\epsilon)\pi_\text{det}$ for some randomly sampled and fixed deterministic policy $\pi_\text{det}$, uniform policy $u$ and $\epsilon=0.8$. The performance evaluates the greedy policy with learned Q-function. New algorithmic variants significantly outperform prior methods.} 
    \label{fig:tabular}
\end{figure}

\paragraph{Can VA-learning underperform Q-learning?}
VA-learning is arguably not \emph{always} more sample efficient than Q-learning. The decomposition $Q(x,a)=V(x)+A(x,a)$, from which VA-learning is derived, assumes that it is useful to share information (i.e.,  $V(x)$) across actions from the same state $x$. When the Q-function gap from a common state $\left|Q(x,a)-Q(x,b)\right|$ is much larger than the gap between different states $\left|V(x)-V(y)\right|$, it is potentially better to learn $Q(x,a)$ and $Q(x,b)$ separately rather than sharing a common value function. Nevertheless, in many practical scenarios, we should expect the utility in sharing values across actions starting from a single state. VA-learning should generally outperform Q-learning, as we show in the following tabular and deep RL experiments.

\subsection{Convergence of VA-learning}
\label{sec:VA-learning-conv}

To understand the behavior of VA-learning more precisely, we consider the expected recursive update that sample-based VA-learning approximates, similar to how Q-learning approximates Bellman recursions. To facilitate the discussion, we define the notation $\mu Q\in\mathbb{R}^\mathcal{X}$ for any $Q\in\mathbb{R}^{\mathcal{X}\times\mathcal{A}}$ as $\mu Q(x)\coloneqq \sum_a \mu(a|x)Q(x,a)$. Abusing the notation a bit, when the context is clear we also use $\mu Q$  to denote a vector in $\mathbb{R}^{\mathcal{X}\times\mathcal{A}}$ with the same value for all actions in a single state $\mu Q(x,a)\coloneqq \mu Q(x)$.

We now introduce the \emph{VA recursion} as a counterpart to the Bellman recursion. For both policy evaluation or control, the VA recursion takes a common form
\begin{align}\tag{VA recursion}
\begin{split}
    V_{t+1} &= \mu \mathcal{T} \left( Q_t - \mu A_t \right), \nonumber \\
    A_{t+1} &= \mathcal{T} \left( Q_t - \mu A_t \right) - V_t,
\end{split}
\end{align} 
with $\mathcal{T}=\mathcal{T}^\pi$ for policy evaluation and $\mathcal{T}=\mathcal{T}^\star$ for control. As we show later, VA-learning is the stochastic approximation to the VA recursion. As a result, the convergence property of VA recursion obviously determines the behavior of VA-learning. We now show that the VA recursion converges to the target Q-function of interest for both policy evaluation and control.

\begin{restatable}{theorem}{theoremconvergence}\label{theorem:convergence} 
(\textbf{Convergence of VA recursion})  For the policy evaluation case, define $V_\mu^\pi(x)\coloneqq \sum_a \mu(a|x)Q^\pi(x,a)$ and $ A_\mu^\pi(x,a)\coloneqq Q^\pi(x,a)-V_\mu^\pi(x)$. Let $C_\mu^\pi=\left\lVert V_0 - V_\mu^\pi \right\rVert_\infty + \left\lVert A_0 - A_\mu^\pi \right\rVert_\infty $ be the initial approximation error. The value and advantage estimates converge geometrically
\begin{align}\tag{policy evaluation}
\begin{split}
    \left\lVert A_t - A_\mu^\pi \right\rVert_\infty &\leq \gamma^{t-1}(1+\gamma)C_\mu^\pi, \nonumber \\ \left\lVert V_t - V_\mu^\pi \right\rVert_\infty &\leq \gamma^t  C_\mu^\pi,
\end{split}
\end{align}
which also implies $\left\lVert Q_t-Q^\pi\right\rVert_\infty=\mathcal{O}(\gamma^t)$. For the control case, we define $V_\mu^\star(x)\coloneqq \sum_a \mu(a|x)Q^\star(x,a)$ and $A_\mu^\star(x,a)\coloneqq Q^\star(x,a)-V_\mu^\star(x)$. Let $C_\mu^\star=\left\lVert V_0 - V_\mu^\star \right\rVert_\infty + \left\lVert A_0 - A_\mu^\star \right\rVert_\infty $ be the initial approximation error. The value and advantage estimates converge geometrically
\begin{align}\tag{optimal control}
\begin{split}
     \left\lVert A_t - A_\mu^\star \right\rVert_\infty &\leq \gamma^{t-1}(1+\gamma)C_\mu^\star, \nonumber \\  \left\lVert V_t - V_\mu^\star \right\rVert_\infty &\leq \gamma^t C_\mu^\star,
\end{split}
\end{align}
which also implies $\left\lVert Q_t-Q^\star\right\rVert_\infty=\mathcal{O}(\gamma^t)$. 
\end{restatable}
\begin{proof}
We show a proof sketch for the policy evaluation case, similar result holds for the control case.
Define $\widetilde{Q}_t=Q_t-\mu A_t$. From the definition of VA recursion, a few calculations show
$
    \widetilde{Q}_{k+1} = \mathcal{T}^\pi\widetilde{Q}_t$.
This implies $\widetilde{Q}_t$ converges to $Q^\pi$ at a geometric rate. Next, since $V_{t+1}=\mu\mathcal{T}^\pi\widetilde{Q}_t$,  we have $V_t\rightarrow V_\mu^\pi$. Finally,  $A_{t+1}=\mathcal{T}^\pi\widetilde{Q}_t-V_t$ implies $A_t\rightarrow Q^\pi-V_\mu^\pi$.
\end{proof}
Intriguingly in general, the converged value function and advantage function differs from the target functions $V_\mu^\pi\neq V^\pi, A_\mu^\pi\neq A^\pi$ (similarly for the control case). The converged value function $V_\mu^\pi$  be understood as the value function obtained by following $\mu$ in the first time step and $\pi$ (resp. $\pi^\ast$ for control). Intuitively, this is because the value updates aggregate across all actions according to $\mu$ without off-policy corrections. Nevertheless, the Q-function estimate constructed from the value and advantage estimate $Q_t(x,a)=V_t(x)+A_t(x,a)$ \textit{does converge} to the target Q-function (resp., $Q^\star$ for control).

\paragraph{Convergence of VA-learning from VA recursion.} Since VA-learning is the stochastic approximation to the VA recursion, the convergent behavior of VA recursion implies that VA-learning should converge too. Indeed, by borrowing the arguments from how TD-learning and Q-learning converge as a result of the convergence of Bellman recursion \citep{watkins1989learning,watkins1992q,jaakkola1993convergence,tsitsiklis1994asynchronous}, we can show VA-learning converges to the target fixed points above given regular assumptions on the data process and learning rate. We provide the detailed results in Appendix~\ref{appendix:VA-learning-converge}.

\subsection{VA-learning with function approximation}

VA-learning is readily compatible with function approximations. In general, consider parameterizing the value function $V_\theta$ and advantage function $A_\phi$ with parameters~$\theta$ and~$\phi$. The Q-function can be computed as $Q_{\theta,\phi}(x,a)\coloneqq V_\theta(x)+A_\phi(x,a)$. Let $\theta^-,\phi^-$ be the target network parameter \citep{mnih2013} which is slowly updated towards $\theta$ and $\phi$. Henceforth, we will focus on the policy evaluation case, similar discussions hold for the control case. Given a transition tuple $(x_t,a_t,r_t,x_{t+1})$, we can construct the back-up value and advantage target based on the tabular VA-learning update,
\begin{align}
\begin{split}\label{eq:VA-learning-function} 
   \widehat{V}(x_t) &= \widehat{Q}^\pi(x_t,a_t) - \gamma A_{\phi^-}(x_{t+1},\mu),  \\
   \widehat{A}(x_t,a_t) &= \widehat{Q}^\pi(x_t,a_t) - \gamma A_{\phi^-}(x_{t+1},\mu) - V_{\theta^-}(x_t),\   
\end{split}
\end{align}
where recall that $\widehat{Q}^\pi(x_t,a_t)=r_t + \gamma Q_{\theta^-,\phi^-}(x_{t+1},\pi)$. The VA-learning update rule can be implemented by minimizing the least square loss function $L_\text{VA}(\theta,\phi)$ defined as
\begin{align}
\begin{split}\label{eq:loss} 
   \frac{1}{2}\left(V_\theta(x_t) -  \widehat{V}(x_t)\right)^2 + \frac{1}{2} \left(A_\phi(x_t,a_t) -  \widehat{A}(x_t,a_t)\right)^2.
\end{split}
\end{align}

\begin{algorithm}[t]
\label{algo:bdueling}
\begin{algorithmic}
\STATE Parameterize value and advantage function $Q_{\theta,\phi}(x,a)=V_\theta(x)+A_\phi(x,a)$. Target network $(\theta^-,\phi^-)$.
\FOR{$t=1,2...$}
\STATE \textbf{Step 1.} Sample transition $(x_t,a_t,r_t,x_{t+1})$.
\STATE \textbf{Step 2.} Learn average behavior policy \begin{align*}
    \psi\leftarrow\psi+\eta\nabla_\psi \log \mu_\psi(a_t|x_t).
\end{align*}
\STATE \textbf{Step 3.} Compute targets $\widehat{V}(x_t),\widehat{A}(x_t,a_t)$ based on Eqn~\eqref{eq:VA-learning-function}, and  update online network parameter using gradient based on VA-learning loss function in Eqn~\eqref{eq:loss}:
\begin{align*}
    (\theta,\phi)\leftarrow (\theta,\phi) -  \eta\nabla_{(\theta,\phi)}L_\text{VA}(\theta,\phi).
\end{align*}
\ENDFOR
\STATE  Output the final Q-function $Q_{\theta,\phi}$.
\caption{VA-learning with function approximation}
\end{algorithmic}
\end{algorithm}

When the behavior policy is unknown and we only have access to samples $(x_t,a_t)$, in order to calculate the back-up target defined in Eqn~\eqref{eq:VA-learning-function}, we need a policy $\mu_\psi$ that keeps track of the average behavior $\mu_\psi(a|x)\approx \mathbb{E}\left[\mathbb{I}\left[x_t=a\right]\;\middle|\;x_t=x\right]$. This can be achieved by maximizing the likelihood $\log \mu_\psi(a|x)$ on observed transitions $(x_t,a_t)$. The full VA-learning algorithm with function approximation is shown in Algorithm 2.

\section{Behavior dueling architecture}
\label{sec:b-dueling}

Thus far, we have showed that the VA-learning advantage function estimate $A_t$ converges to $A_\mu^\pi$ for policy evaluation (resp.  $A_\mu^\star$ for control). By definition, such advantage functions satisfy the following \emph{zero-mean} property
\begin{align}
\begin{split} \label{eq:zero-mean-mu}
    A_\mu^\pi(x, \mu) &\coloneqq \sum_a \mu(a|x) A_\mu^\pi(x,a) = 0 \\
    A_\mu^\ast(x,\mu) &\coloneqq \sum_a \mu(a|x) A_\mu^\star(x,a) = 0.
\end{split}
\end{align}
At any finite iteration $t$, the estimate $A_t$ does not necessarily satisfy the above property. Since we know the zero-mean property that the converged value of $A_t$ satisfies, it is tempting to enforce such a property as a constraint on $A_t$, which does not change the fixed point of the update. In the function approximation case, such a zero-mean constraint might be a useful inductive bias for parameterizing the advantage function. For example, we let $f_\phi(x,a)$ be an unconstrained function, and define its average over actions $f_\phi(x,\mu)\coloneqq\sum_a \mu(a|x) f_\phi(x,a)$. We parameterize the zero-mean advantage function as follows
\begin{align}
    A_\phi(x,a)\coloneqq f_\phi(x,a)-f_\phi(x,\mu),\label{eq:behavior-dueling}
\end{align} 
such that $A_\phi(x,\mu)\coloneqq \sum_a \mu(a|x)A_\phi(x,a)=0$.
We call the above parameterization \emph{behavior dueling} due to its close connections to the dueling architecture \citep{wang2016} and the fact that we parameterize the advantage to be zero-mean under the behavior policy. The regular dueling architecture is a special case when $\mu$ is uniform.
The full-fledged Q-learning with behavior dueling algorithm is shown in Algorithm 2. 

\subsection{Connections between VA-learning and Q-learning with behavior dueling}

Our key insight is that VA-learning bears close conceptual connections to regular TD-learning (or Q-learning) with behavior dueling parameterization. With behavior dueling, TD-learning or Q-learning might benefit from the value sharing parameteirzation and the inductive bias for learning advantage functions. To better see the connections, note that the TD-learning algorithm minimizes the least square loss function with respect to the parameterized Q-function $Q_{\theta,\phi}$:
\begin{align}
    L_\text{QL}(\theta,\phi) = \frac{1}{2} \left(Q_{\theta,\phi}(x_t,a_t) - \widehat{Q}^\pi(x_t,a_t)\right)^2,\label{eq:bdueling-q}
\end{align}
where $\widehat{Q}^\pi(x_t,a_t) = r_t + \gamma Q_{\theta^-,\phi^-}(x_{t+1},\pi)$ is the one-step back-up target.
With behavior dueling, $Q_{\theta,\phi}(x,a)=V_\theta(x)+A_\phi(x,a)$ and $A_\phi(x,a)=f_\phi(x,a)-f_\phi(x,\mu)$. We examine the gradient of Q-learning loss function $L_\text{QL}(\theta,\phi)$ with respect to the value parameter $\nabla_\theta L_\text{QL}(\theta,\phi)$ is 
\begin{align*}
    \left(V_\theta(x_t) - \left(\widehat{Q}^\pi(x_t,a_t) - A_\phi(x_t,a_t) \right)\right) \nabla_\theta V_\theta(x_t).
\end{align*}
We can interpret the gradient for value parameter $\theta$ as updating the value function $V_\theta$ so as to better fit the value function target $ \widehat{Q}^\pi(x_t,a_t) - A_\phi(x_t,a_t)$. This echos with the value updates in VA-learning that which aim to fit a value function target (see Section~\ref{sec:VA-learning}). 

Regarding the advantage updates, there is a subtle difference between the advantage updates of VA-learning vs. behavior dueling. In a nutshell, this is because VA-learning carries out separate updates for each advantage function $A(x,a)$. In contrast, behavior dueling couples advantage updates for different actions due to the parameterization $A(x,a)=f(x,a)-f(x,\mu)$; as a result, when action $a\neq b$ is taken, the advantage function $f(x,b)$ is updated as well. Despite the difference, both updates bear the interpretations of fitting the advantage components of the Q-function. Such interpretations imply that the motivational example (Section~\ref{sec:VA-learning-why}) which illustrates that the benefits of VA-learning should intuitively apply to behavior dueling as well, as we will verify empirically. We present a more complete discussion of such connections between VA-learning and behavior dueling in Appendix~\ref{appendix:why-bdueling}.

\begin{algorithm}[t]
\label{algo:bdueling}
\begin{algorithmic}
\STATE Behavior dueling  $Q_{\theta,\phi}(x,a)=V_\theta(x)+A_\phi(x,a)$ with parameterization $A_\phi(x,a)=f_\phi(x,a)-f_\phi(x,\mu_\psi)$. Target network $(\theta^-,\phi^-)$.
\FOR{$t=1,2...$}
\STATE \textbf{Step 1.} Sample transition $(x_t,a_t,r_t,x_{t+1})$.
\STATE \textbf{Step 2.} Learn average behavior policy \begin{align*}
    \psi\leftarrow\psi+\eta\nabla_\psi \log \mu_\psi(a_t|x_t).
\end{align*}
\STATE \textbf{Step 3.} Compute back-up target
\begin{align*}
    \widehat{Q}^\pi(x_t,a_t)= r_t + \gamma Q_{\theta^-,\phi^-}(x_{t+1},\pi).
\end{align*}
and update online network parameter using gradient based on Eqn~\eqref{eq:bdueling-q}: $(\theta,\phi)\leftarrow (\theta,\phi) -  \eta\nabla_{(\theta,\phi)}L_\text{QL}(\theta,\phi)$.
\ENDFOR
\STATE  Output the final Q-function $Q_{\theta,\phi}$.
\caption{Q-learning with behavior dueling}
\end{algorithmic}
\end{algorithm}

\subsection{Why behavior dueling is better than dueling}
We can understand the dueling architecture \citep{wang2016} as a special case of behavior dueling assuming $\mu$ is uniform. Such an implicit assumption can be useful when $\mu$ is indeed close to uniform, so that there is no need to parameterize an additional behavior policy $\mu_\psi$ to learn.
However, when the behavior policy deviates from the uniform policy, learning $\mu_\psi\approx \mu$ seems critical to improved performance. In a few practical setups, we usually find behavior dueling to outperform uniform dueling, as we will demonstrate in both tabular and some large-scale deep RL settings.

In light of the discussion in Section~\ref{sec:VA-learning-why}, both behavior and uniform dueling entail sharing information across actions, so why does the former perform better? We hypothesize that this is because behavior dueling entails a better value sharing between actions, as it is adapted to the behavior policy. Consider the dueling parameterization $A_\nu(x,a)=f(x,a)-f(x,\nu)$ with distribution $\nu$. We are interested in minmizing unshared information $A_\mu(x,a)$ across actions, as characterized by the squared norm
\begin{align*}
    \min_\nu \sum_a \mu(a|x) A_\nu(x,a)^2
\end{align*}
In general, the minimizing distribution is $\nu=\mu$, i.e., the behavior policy. Since behavior dueling at $\nu=\mu$ minimizes the unshared components of Q-functions, it can be interpreted as maximizing the shared components, leading to faster downstream learning.
We provide a more precise argument in Appendix~\ref{appendix:why-bdueling} with experiment ablations.

\section{Discussion of prior work}\label{sec:discuss}

We discuss the relation between VA-learning and a few lines of related work in RL.

\paragraph{Advantage learning.} Despite the  similarity in names, VA-learning differs from advantage learning \citep{baird1993advantage,baird1995residual} in critical ways. In a nutshell, VA-learning still aims to learn the original Q-function $Q^\pi$ (or $Q^\ast$ in the control case), whereas advantage learning learns to increase the value gaps between actions. Specifically, given a transition $(x_t,a_t,r_t,x_{t+1})$, advantage learning for optimal control can be understood as the following back-up target for $Q_t(x_t,a_t)$ \citep{bellemare2016increasing,kozuno2019gap}
\begin{align*}
    \widehat{\mathcal{T}}^\pi Q_t(x_t,a_t) + \beta \left(Q_t(x_t,a_t) - \sum_a \pi(a_t|x_t) Q_t(x_t,a_t)\right)
\end{align*}
where $\pi$ is the greedy policy for the control case.
The above back-up operation is gap-increasing, in that it enlarges the difference between converged Q-functions at different actions. For example, in the policy evaluation case the Q-function converges to $V^\pi+\frac{1}{1-\beta}A^\pi$. As $\beta\rightarrow 1$, the Q-function gap between actions increases. 

Compared to gap-increasing operators, a subtle technical difference is that VA-learning constructs the back-up target by subtracting the average Q-function under behavior policy $\mu$ instead of target policy $\pi$; and at the next state $x_{t+1}$, instead of the current state $x_t$. This ensures that VA-learning still retains $Q^\pi$ as the fixed point. An interesting future direction would be to combine VA-learning with the gap-increasing learning.

\paragraph{Direct advantage learning.} With a similar motivation as VA-learning, \citet{pan2021direct} proposed to learn advantage functions directly from Monte-Carlo returns, based on the variance-minimization property of advantage function. Their approach is thus far constrained to the on-policy case, and does not allow for bootstrapping out-of-the-box. An interesting direction would be to study the combination of such an approach with VA-learning.

\paragraph{RL with over-parameterized linear function approximation.} The tabular dueling parameterization can be understood as a special case of over-parameterized linear function approximation \citep{sutton1998}. Here, \emph{over-parameterized} refers to the fact that dueling introduces an extra degree of freedom to learning Q-functions. We have demonstrated empirically that this extra degree of freedom allows for value sharing across actions, and usually helps speed up convergence. Interesting open questions include the study of off-policy stability of dueling parameterization.

\section{Experiments}\label{sec:exp}

We start with experiments on tabular MDPs, to understand the improved sample efficiency of VA-learning over Q-learning. Then we evaluate the impacts of VA-learning and behavior dueling in deep RL settings.

\subsection{Tabular MDP experiments}

In Figure~\ref{fig:tabular-eps}, we compare four algorithmic variants with behavior policy $\mu$ defined as $\mu=\epsilon u + (1-\epsilon)\pi_\text{det}$ on a family of randomly generated tabular MDPs. Here, $\pi_\text{det}$ is a randomly sampled deterministic policy, $u$ is the uniform policy and $\epsilon \in [0,1]$ is the mixing coefficient. We calculate the final performance of each algorithm until convergence, and show the mean and standard deviation across $20$ independent runs. See Appendix~\ref{appendix:exp} for more details on the MDP details.

For a wide range of values of $\epsilon$, VA-learning and Q-learning with behavior dueling outperform other baselines significantly. The performance gap seems the most profound when $\epsilon\approx 0$ as $\mu$ deviates the most from uniform. In this case, we speculate that since the behavior dueling architecture makes an inaccurate implicit assumption on the behavior policy, Q-learning with uniform dueling perform poorly as regular Q-learning. As $\epsilon$ increases, the performance gap decreases. When $\epsilon\rightarrow 1$ and $\mu$ is close to a uniform policy, uniform dueling catches up with the two new algorithms. However, there is still a statistically significant gap between regular Q-learning and other algorithms, implying a consistent benefit of VA-learning and its derived Q-learning variants (behavior dueling) over regular Q-learning.

\begin{figure}[t]
    \centering
    \includegraphics[keepaspectratio,width=.42\textwidth]{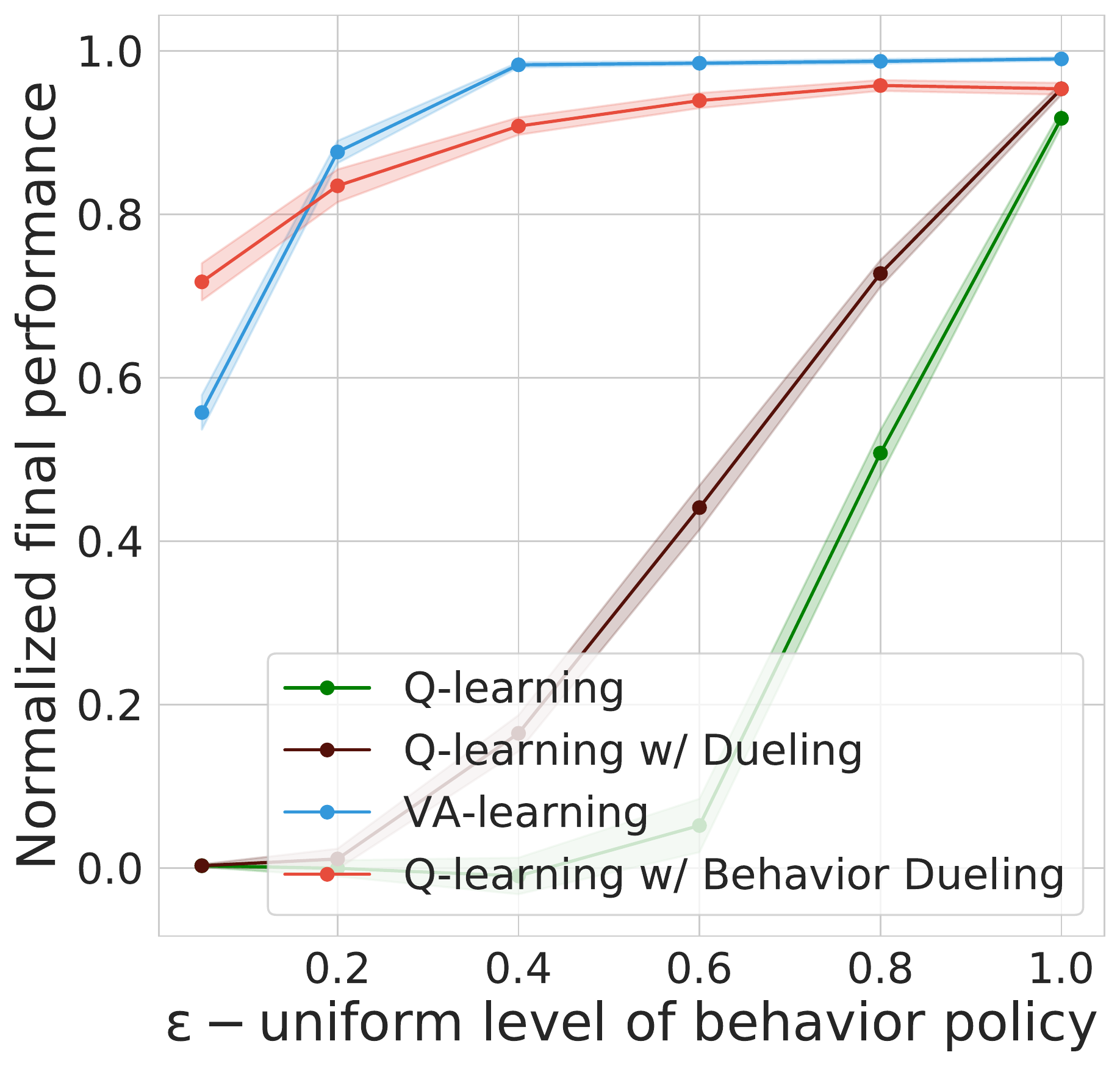}
    \caption{Comparing different algorithmic variants in tabular MDPs with fixed behavior policy $\mu=\epsilon u + (1-\epsilon)\pi_\text{det}$ for some randomly sampled and fixed deterministic policy $\pi_\text{det}$, uniform policy $u$ and varying degree of $\epsilon$ ($x$-axis). As $\epsilon\rightarrow 1$ and $\mu$ approaches uniform, Q-learning with dueling architecture catches up in performance with behavior dueling and VA-learning.}
    \label{fig:tabular-eps}
\end{figure}

    \label{fig:tabular-sample}

\subsection{Deep reinforcement learning experiments}

We now evaluate the effects of VA-learning and its variants in large-scale deep RL environments. We use the DQN agent \citep{mnih2013} as the baseline agent and use the Atari 57 game suite as the test bed \citep{bellemare2013arcade}. Throughout, we report the interquartile mean (IQM) score across multiple random seeds for all algorithmic variants that train for $200$M frames \citep{agarwal2021deep}. We compare VA-learning, behavior dueling, dueling \citep{wang2016} and baseline Q-learning. All variants share the same architecture and hyper-parameters wherever possible. 

The behavior policy $\mu$ is $\epsilon$-greedy with respect to the Q-function network $Q_{\theta,\phi}$. Since both the exploration rate $\epsilon$ and Q-function  $ Q_{\theta,\phi}$ slowly changes over time, the behavior policy $\mu$ changes too. VA-learning and behavior dueling trains an additional average behavior policy $\mu_\psi(a|x)$ to approximate the average behavior policy across the entire training history. By default, to improve performance, the baseline Q-learning agent implements $n$-step bootstrapping and computes back-up targets based on partial trajectories of length $n$. VA-learning can be easily adapted accordingly, see Appendix~\ref{appendix:exp} for more details.

\begin{figure}[t]
    \centering
    \includegraphics[keepaspectratio,width=.42\textwidth]{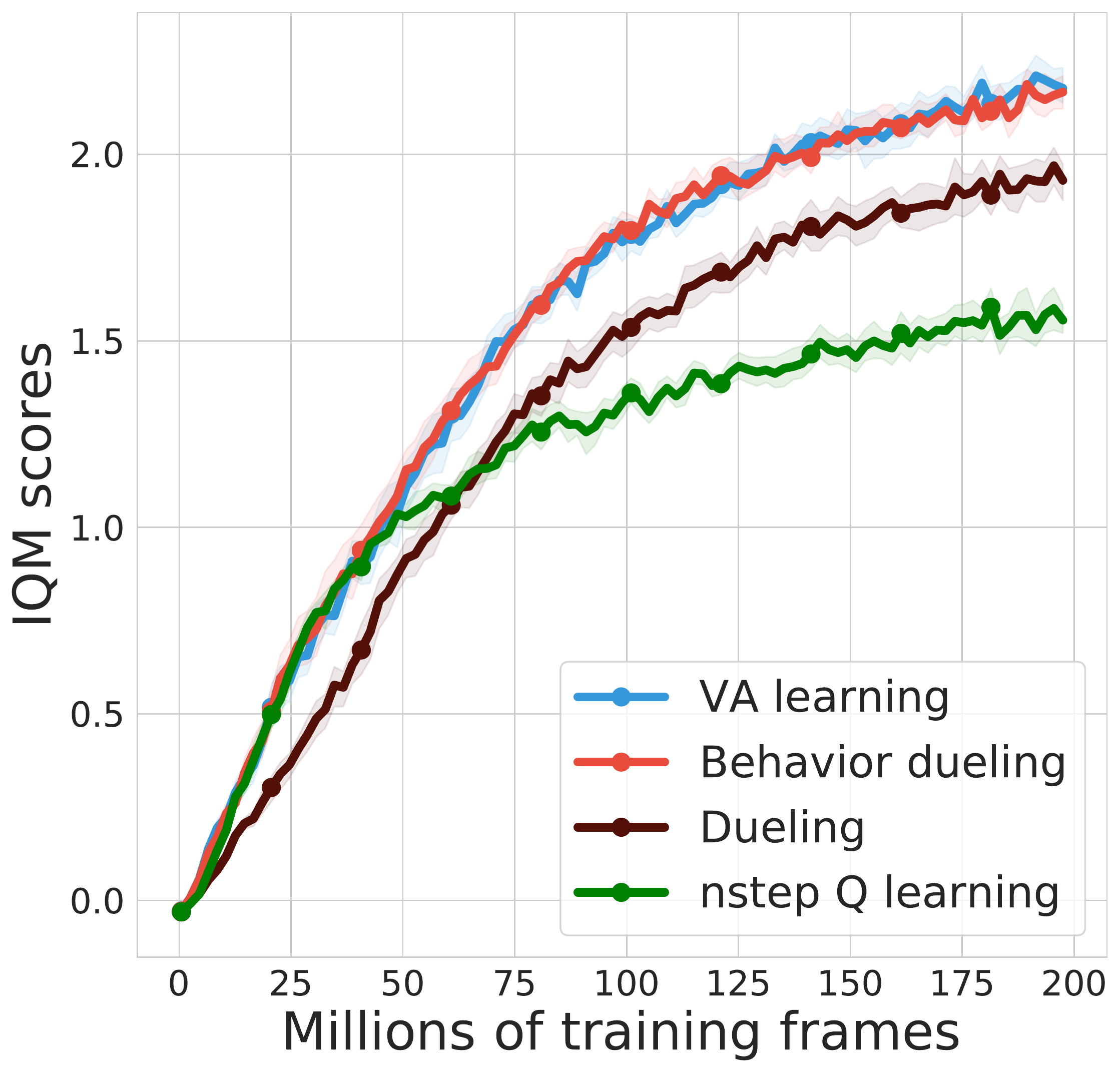}
    \caption{Comparing algorithmic variants implemented with full Atari action set. VA-learning and behavior dueling are significantly better than the uniform dueling architecture, which further improves over the $n$-step Q-learning baseline. Compared to the standard Atari setup in Figure~\ref{fig:epsilon-dqn}(b), the performance of VA-learning and behavior dueling does not degrade.}
    \label{fig:allaction-dqn}
\end{figure}

\paragraph{Network architecture.} The baseline DQN agent network consists of a \emph{torso} convolutional network which processes the input image $x$ into an embedding $g(x)$, and a \emph{head} MLP network which takes the embedding and outputs the Q-function $Q_\theta(g(x),a)$. The dueling architecture parameterizes a separate value \emph{head} network $V_\theta\left(g(x)\right)$ and advantage \emph{head} network $A_\phi(g(x),a)$. In behavior dueling and VA-learning, the behavior policy is parameterized as a policy \emph{head} network that outputs a distribution over actions $\mu_\psi\left(a|g(x)\right)$. Throughout  experiments, we design the policy head to share the same torso as the other value heads, but its gradient does not update the torso parameters. This design choice ensures that the loss function for learning average behavior policy does not shape the embedding. Hence, any resulting empirical gains   can be more convincingly attributed to the improvements of VA-learning over Q-learning. See Appendix~\ref{appendix:exp} for more comprehensive details.

\paragraph{Full action set Atari.}
We focus on a variant of the Atari game suite with \emph{full action set}, where the agent has access to a total of $|\mathcal{A}|=18$ actions including potentially many actions with no effect. Thus far by default, agents are trained with the restricted action set which makes learning  easier (e.g., for Pong the reduced action set has $|\mathcal{A}|=3$ actions). 

In Figure~\ref{fig:allaction-dqn}, we compare DQN agent variants with the full action set. Almost all algorithms can reach a similar asymptotic performance as with the restricted action set (Figure~\ref{fig:epsilon-dqn}(b)) but the learning speed is generally slowed down. VA-learning and behavior dueling are the least impacted by the increased action set. The dueling architecture slows down more significantly, with the performance margins against VA-learning enlarged over time. As discussed in Section~\ref{sec:b-dueling}, the dueling architecture can be understood as imposing an implicit uniform assumption on the behavior policy. When the action space is large, such an assumption is more easily violated as the agent is much more likely to take certain actions than others over time. Such a comparison highlights the practical importance of using the behavior policy to carry out the average of advantage function, which is the design principle of VA-learning.

For the restricted Atari game setting where for each game only a small subset of full action sets is provided to the agent, we observe that behavior dueling and VA-learning also deliver improvements over dueling and Q-learning baselines. See Appendix~\ref{appendix:exp} for more results and ablation study in the deep RL setting.

\section{Conclusion}
In this work, we have developed VA-learning  as an alternative value-based RL algorithm to the classic Q-learning. 
We have discussed a few important theoretical aspects of VA-learning, and how it can be implemented with function approximations. With the extra degree of freedom in place, VA-learning aims to learn a value function and advantage function that is adapted to the behavior policy. Compared to Q-learning, VA-learning makes more efficient use of finite samples and enjoys better empirical performance in both tabular and deep RL settings. VA-learning also inspires the behavior dueling architecture, which generalizes dueling as a special case, and potentially explains why such a seemingly simple architecture change helps improve DQN.

\paragraph{Acknowledgements.} We especially thank Bruno Scherrer for providing valuable feedback that identifies typos in the proof of a previous version of the paper.

\bibliography{main}
\bibliographystyle{plainnat}

\newpage
\onecolumn

\begin{appendix}

\section*{\centering APPENDICES: VA-learning as a more efficient alternative to Q-learning}

\section{Proof of theoretical results}
\label{appendix:proof}

Taking policy evaluation as an example, we first show that the VA-learning update
\begin{align*}
\begin{split}
    V_{t+1}(x_t) &\overset{\alpha_t}{\leftarrow} \widehat{\mathcal{T}}^\pi Q_t(X,A) - A_t(x_{t+1},\mu), \nonumber \\
    A_{t+1}(x_t,a_t) &\overset{\alpha_t}{\leftarrow} \widehat{\mathcal{T}}^\pi Q_t(X,A) - A_t(x_{t+1},\mu) - V_t(x_t). 
\end{split}
\end{align*}
is reduced to the above VA recursion in expectation,
\begin{align*}
\begin{split}
    V_{t+1} &= \mu \mathcal{T}^\pi \left( Q_t - \mu A_t \right), \nonumber \\
    A_{t+1} &= \mathcal{T}^\pi \left( Q_t - \mu A_t \right) - V_t.
\end{split}
\end{align*} 
We first consider the value function update. Conditional on the initial state $X$, we take an expectation over the random variables
\begin{align*}
    a_t \sim \mu (\cdot|x_t), r_t \sim P_R(\cdot|x_t,a_t), x_{t+1}\sim P(\cdot|x_t,a_t).
\end{align*}
This leads to
\begin{align*}
    \mathbb{E}\left[\widehat{\mathcal{T}}^\pi Q_t(x_t,a_t) - A_t(x_{t+1},\mu) \;\middle|\; x_t\right] = \sum_a \mu(a|x_t) \mathcal{T}^\pi\tilde{Q}_t(x_t,a),
\end{align*}
where $\tilde{Q}_t\coloneqq Q_t-\mu A_t$.
In our notation, this is equivalent to $\mu  \mathcal{T}^\pi(X)$. This means the value function update in expectation is indeed $V_{t+1}(x)= \sum_a \mu(a|x) \mathcal{T}^\pi\tilde{Q}_t(x,a)$. With the same set of argument, we can show the case for the advantage function update too.

\theoremconvergence*
\begin{proof}
We first examine the policy evaluation case.
Define $\widetilde{Q}_t=Q_t-\mu A_t$. From the definition of VA recursion, we have
\begin{align*}
    \widetilde{Q}_{t+1} &= V_{t+1}+A_{t+1}-\mu A_{t+1} \\ 
    &= \mu \mathcal{T}^\pi(Q_t-\mu A_t) + \mathcal{T}^\pi(Q_t-\mu A_t) - V_t - \mu \left(\mathcal{T}^\pi(Q_t-\mu A_t) - V_t\right) \\
    &= \mathcal{T}^\pi(Q_t-\mu A_t)  \\
    &= \mathcal{T}^\pi \widetilde{Q}_t.
\end{align*}
where we have exploited the fact that $\mu V_t=V_t$. 
The above equality implies $\widetilde{Q}_t$  converges to $Q^\pi$ at a geometric rate, since the operator $\mathcal{T}^\pi$ has $Q^\pi$ as the unique fixed point and is $\gamma$-contractive. Formally, we have
\begin{align*}
    \left\lVert \widetilde{Q}_t - Q^\pi \right\rVert_\infty \leq \gamma^t \left\lVert \widetilde{Q}_0 - Q^\pi \right\rVert_\infty \leq_{(a)} \gamma^t \underbrace{\left(\left\lVert A_0 - A_\mu^\pi \right\rVert_\infty + \left\lVert  V_0 - V_\mu^\pi \right\rVert_\infty \right)}_{\eqqcolon C_\mu^\pi}.
\end{align*}
Here, (a) follows from the application of triangle inequality and the fact that $Q^\pi = A_\mu^\pi + V_\mu^\pi$. Now, we can write
\begin{align*}
    \left\lVert V_t - V_\mu^\pi\right\rVert_\infty =  \left\lVert \mu \mathcal{T}^\pi \widetilde{Q}_{k-1} - V_\mu^\pi\right\rVert_\infty = \left\lVert \mu \widetilde{Q}_{k} - V_\mu^\pi\right\rVert_\infty \leq \gamma^t C_\mu^\pi.
\end{align*}
Finally, we consider the advantage function. 
\begin{align*}
    \left\lVert A_t - A_\mu^\pi \right\rVert_\infty = \left\lVert \mathcal{T}^\pi \widetilde{Q}_{k-1} - V_t - A_\mu^\pi \right\rVert_\infty &= \left\lVert \widetilde{Q}_t - V_{k-1} - A_\mu^\pi \right\rVert_\infty \\
    &\leq_{(a)} \left\lVert \widetilde{Q}_t - Q^\pi \right\rVert_\infty + \left\lVert V_{k-1} - V_\mu^\pi \right\rVert_\infty \\
    &\leq \gamma^{t-1}(1+\gamma) C_\mu^\pi,
\end{align*}
where (a) follows from the application of triangle inequality.
This concludes the proof for policy evaluation. For optimal control, the same set of argument applies thanks to the fact that $\mathcal{T}$ is $\gamma$-contractive with $Q^\star$ as the unique fixed point.
\end{proof}

\section{Convergence of VA-learning}
\label{appendix:VA-learning-converge}

We now present results on the convergence of VA-learning under stochastic approximations. Recall that upon observing the sample $(x_t,a_t,r_t,x_{t+1})$, VA-learning carries out the following update ,
\begin{align*}
\begin{split}
    V_{t+1}(x_t) &\overset{\alpha_t}{\leftarrow} \widehat{\mathcal{T}} Q_t(X,A) - A_t(x_{t+1},\mu), \nonumber \\
    A_{t+1}(x_t,a_t) &\overset{\alpha_t}{\leftarrow} \widehat{\mathcal{T}} Q_t(X,A) - A_t(x_{t+1},\mu) - V_t(x_t),
\end{split}
\end{align*}
where $\widehat{\mathcal{T}} Q_t(X,A)$ is the one-sample stochastic approximation to $\mathcal{T}Q_t(X,A)$ for optimal control and $\mathcal{T}^\pi Q_t(X,A)$ for policy evaluation.
We consider a more restrictive setup, where from each state $x$, we sample action $A_x\sim \mu(\cdot|x)$, and observe the corresponding immediate reward $R_x\sim P_R(\cdot|x,A)$ and next state transition $X_x\sim P(\cdot|x,A)$, where the subscripts are meant to distinguish between samples from different state $x\in\mathcal{X}$. The update is carried out across all states simultaneously, $\forall x\in\mathcal{X}$,
\begin{align}
\begin{split}\label{eq:VA-learning-sa}
    V_{t+1}(x_t) &\overset{\alpha_t}{\leftarrow} \widehat{\mathcal{T}} Q_t(x_t,a_t) - A_t(x_{t+1},a_x^\mu),  \\
    A_{t+1}(x_t,a_t) &\overset{\alpha_t}{\leftarrow} \widehat{\mathcal{T}} Q_t(x_t,a_t) - A_t(x_{t+1},a_x^\mu) - V_t(x_t).
\end{split}
\end{align}
The formal results are stated as follows.
\begin{restatable}{theorem}{theoremconvergenceva}\label{theorem:convergenceva} 
(\textbf{Convergence of VA-learning}) Under the assumption $\sum_{t=0}^\infty \alpha_t=\infty$ and $\sum_{t=0}^\infty \alpha_t^2 \leq C < \infty$ where $C$ is some finite constant, then the above update in Eqn~\eqref{eq:VA-learning-sa} leads to almost sure convergence of the iterates. Concretely, 
\begin{align*}
    V_t(x)\rightarrow V_\mu^\pi, A_t(x)\rightarrow A_\mu^\pi,\forall (x,a)\in\mathcal{X}\times\mathcal{A}
\end{align*}
almost surely for policy evaluation and 
\begin{align*}
    V_t(x)\rightarrow V_\mu^\star, A_t(x,a)\rightarrow A_\mu^\star,\forall (x,a)\in\mathcal{X}\times\mathcal{A}
\end{align*}
for optimal control.
\end{restatable}

The proof is a straightforward extension of classic proof technique to show the stochastic approximation convergence of Q-learning and TD-learning \citep{watkins1992q}.

\section{Experiment details and extra results}
\label{appendix:exp}

We provide further details on the tabular and deep RL experiments in the main paper.

\subsection{Tabular experiments}

All tabular experiments in the paper are carried out on randomly generated MDPs with $|\mathcal{X}|=20$ states, $|\mathcal{A}|=5$ actions and discount factor $\gamma=0.99$. The transition matrix $p(\cdot|x,a)$ is generated from a Dirichlet distribution with parameter $(\alpha,\alpha...\alpha)$ for $\alpha=0.5$. For the control case, the behavior policy is fixed and constructed as $\mu=\epsilon u + (1-\epsilon)\pi_\text{det}$ for some randomly sampled and fixed deterministic policy $\pi_\text{det}$, uniform policy $u$ and $\epsilon\in[0,1]$. By adjusting $\epsilon$, we can assess different algorithms' robustness to the level of stochasticity in the behavior policy. In the policy evaluation case, $\mu$ is set to be uniform and equivalently $\epsilon=1$. The trajectories are collected starting from the same state $x=0$ and under behavior policy $\mu$. Trajectories are truncated at length $T=\text{int}\left(\frac{2}{1-\gamma}\right)$ where $\text{int}(x)$ denotes the closest integer to $x$. By default, $N=20$ are collected for each experiment.

In the control case, the performance is calculated as $Q^{\pi_t}$ where $\pi_t$ is the greedy policy with learned Q-function $Q_t$. Recall that for VA-learning, $Q_t(X,A)=V_t(x)+A_t(x,a)$. In plots, we show the average value of $Q^{\pi_t}$ uniformly across all states. In the policy evaluation case, we calculate $\left\lVert Q_t - Q^\pi\right\rVert_2$ where $\pi$ is the randomly chosen deterministic policy. 

Figure~\ref{fig:tabular}, we demonstrate how VA-learning improves over Q-learning in the policy evaluation case, by measuring the advantage estimation error $\left\lVert \hat{A}_t-A^\pi\right\rVert_2$ over time. For VA-learning, $\hat{A}_t=A_t$; for Q-learning, $\hat{A}_t=Q_t-\pi Q_t$.

\paragraph{Gradient descent updates.} Throughout tabular experiments, we implement updates as gradient descents on a certain properly defined loss functions. We always adopt tabular parameterizations of $Q_t,V_t$ and $a_t$. For Q-learning, the loss function is implemented as in $L_\text{QL}$; for the dueling architecture, the loss function is the same as Q-learning but with dueling parametrization on $Q_t(X,A)$; for VA-learning, the loss function is implemented as in $L_\text{VA}$. At each update, the gradient is averaged across all collected trajectories so as to avoid additional randomness in the update. The learning rate is set as a constant $\alpha_t=0.1$. The target parameter is copied to be the online parameter every $\tau=10$ updates.

\paragraph{Extra results on Q-function error.}

As complementary results to Figure~\ref{fig:tabular}, we demonstrate how VA-learning improves over Q-learning in the policy evaluation case in Figure~\ref{fig:tabular-adv}(b). We measure the Q-function estimation error $\left\lVert \hat{Q}_t-Q^\pi\right\rVert_2$ over time. For VA-learning, $\hat{Q}_t=A_t+V_t$; for Q-learning, $\hat{Q}_t=Q_t$. Notably, VA-learning achieves a slightly faster decaying rate of the approximation error compared to Q-learning. 

\begin{figure}[t]
    \centering
    \subfigure[Advantage error]{\includegraphics[keepaspectratio,width=.42\textwidth]{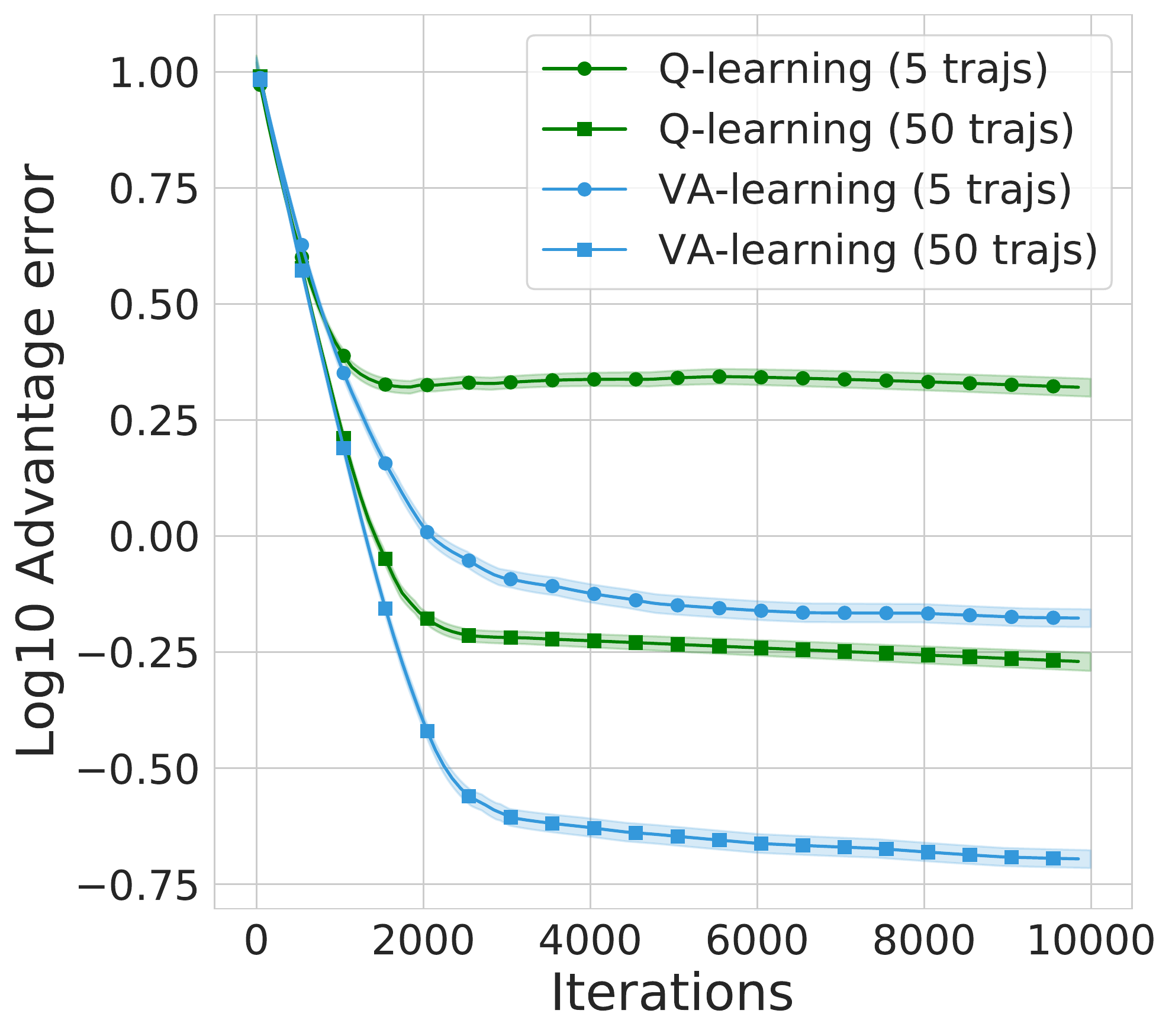}}
    \subfigure[Q-function error]{\includegraphics[keepaspectratio,width=.42\textwidth]{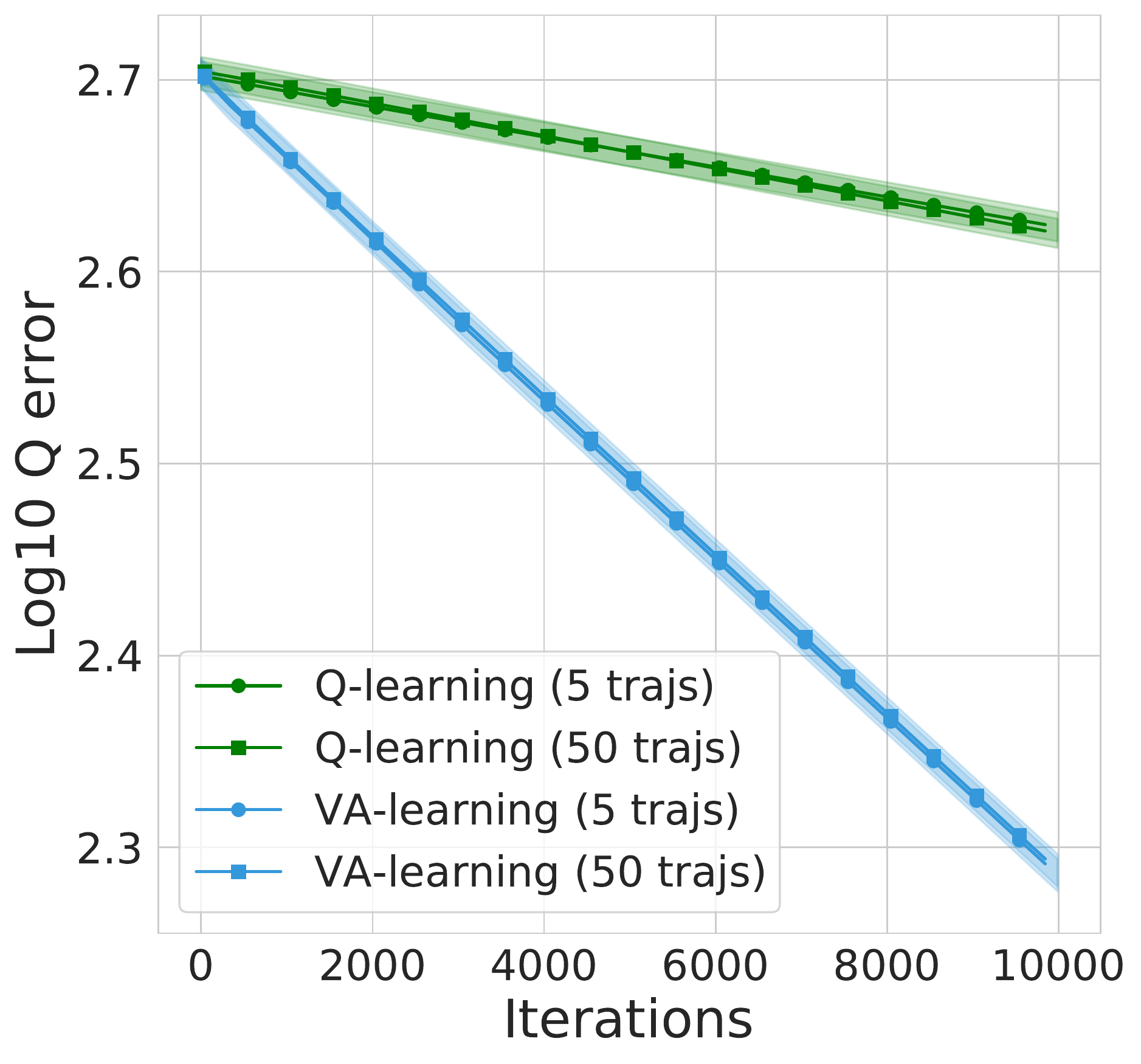}}
    \caption{(a) Comparing VA-learning (Section~\ref{sec:VA-learning}) with Q-learning for tabular policy evaluation. We evaluate a target policy $\pi$ formed as a convex combination of a deterministic policy and uniform policy, using a fixed number of trajectories data collected under uniform policy. The $y$-axis shows the approximation error to the advantage $\left\lVert \widehat{A}_t-A^\pi\right\rVert_2$ at each iteration $k$. Given any data budget, VA-learning obtains more accurate approximations to the advantage function compared to Q-learning. (b) The same setup as before. The $y$-axis shows the approximation error to the Q-function $\left\lVert \widehat{Q}_t-Q^\pi\right\rVert_2$ at each iteration $k$. Given any data budget, VA-learning obtains a slightly faster rate of approximating the Q-function compared to Q-learning. }
    \label{fig:tabular-adv}
\end{figure}

\subsection{Deep RL experiment details}

All the deep RL experiments use the DQN agent \citep{mnih2013} as the baseline agent and use the Atari 57 game suite as the test bed \citep{bellemare2013arcade}. To improve the performance, we apply double Q-learning \citep{van2016deep} and $n$-step bootstrapping in general. Throughout, we use $n=5$; if $n$ is smaller, overall DQN does not benefit fully from multi-step bootstrapping; if $n$ is larger, the performance can suffer due to lack of off-policy corrections. The agent adopts most architecture and hyper-parameters as reported in \citep{mnih2013}. Our agents are all based on the reference implementation in \citep{quandqn}. 

We report the interquartile mean (IQM) score across multiple random seeds for all algorithmic variants that train for $200$M frames \citep{agarwal2021deep}. See the reference for specific procedures for calculating the IQM score and the bootstrapped confidence intervals.

\paragraph{Multi-step bootstrapping, $n$-step Q-learning and VA-learning.} Multi-step bootstrapping usually improves practical performance of deep RL algorithms \citep{hessel2018rainbow}. In $n$-step Q-learning, the agent samples a partial trajectory $\left(X_{0:t},A_{0:t-1},R_{0:t-1}\right)$ starting from state-action pair $(x_t,a_t)$, and constructs the policy evaluation back-up target with target policy $\pi$,
\begin{align*}
    \widehat{\mathcal{T}}^\pi Q_t(x_t,a_t) = \sum_{t=0}^{n-1} \gamma^t r_t + \gamma^n Q_t(x_n, \pi)
\end{align*}
For control, the back-up target is 
\begin{align*}
    \widehat{\mathcal{T}}^\pi Q_t(x_t,a_t) = \sum_{t=0}^{n-1} \gamma^t r_t + \gamma^n \max_a Q_t(x_n, a).
\end{align*}
Finally, $n$-step Q-learning carries out the update $Q_{t+1}(x_t,a_t)  \overset{\alpha_t}{=} \widehat{\mathcal{T}}^\pi Q_t(x_t,a_t)$. VA-learning can adapt to $n$-step bootstrapping as follows: 
\begin{align*}
    V_{t+1}(x_0) &\overset{\alpha_t}{=} \widehat{\mathcal{T}}^\pi Q_t(x_t,a_t) - A_t(x_n,\mu) \\
    A_{t+1}(x_t,a_t) &\overset{\alpha_t}{=} \widehat{\mathcal{T}}^\pi Q_t(x_t,a_t) - A_t(x_n,\mu) - V_t(x_0).
\end{align*}
When $n=1$, the above recovers the VA-learning introduced in the paper as a special case. 

\paragraph{Details on network architecture.} We use the standard DQN architecture specified in \citep{mnih2013}. As described in the main paper, given $4$ stacked frames from the Atari game as input state $x$, the \emph{torso} convolutional neural network processes the image into an embedding $g(x)$. The DQN agent parameterizes a value \emph{head}, which is a MLP that takes $g(x)$ and produces $|\mathcal{A}|$ scalar outputs, each corresponding to a Q-function prediction $Q_\theta(g(x),a)$. The dueling architecture, behavior dueling and VA-learning all parameterize a separate value \emph{head} with one scalar output $V_\theta\left(g(x)\right)$, and an advantage \emph{head} with $|\mathcal{A}|$ scalar outputs $A_\phi\left(g(x),a\right)$. 
VA-learning and behavior dueling further parameterize a policy \emph{head} network $\mu_\psi\left(g(x),a\right)$ that outputs a probability distribution over actions.
The overall Q-function is then produced as \begin{align*}
    Q_{\theta,\phi}\left(g(x),a\right)=V_\theta\left(g(x)\right)+A_\phi(x,a) - \sum_a \mu_\psi\left(g(x),a\right)A_\phi\left(g(x),a\right).
\end{align*}
The torso parameters $g$ are trained with the Q-learning or VA-learning loss function. We put a stop gradient on the torso embedding $g(x)$ when calculating the learned behavior policy distribution, such that the behavior learning loss function does not impact $g$.

\paragraph{Tuning learning rate.} Learning rate is the only hyper-parameter we tune across DQN agent variants. All agents use the RMSProp optimizer \citep{hinton2012neural}. By default, one-step DQN agent uses the learning rate $\beta=2.5\cdot 10^{-4}$. When using $n$-step Q-learning with $n=5$, we find the learning rate is best set smaller to be at $5\cdot 10^{-5}$. When doing VA-learning, behavior dueling and uniform dueling, we find it improves performance further by reducing the learning rate more, to $1.5\cdot 10^{-5}$. 
All learning rates are found by grid search: we start with the default learning rate $\beta$ of DQN, and experiment on a subset of games whether setting learning rates at $\frac{1}{3}\beta$ or $3\beta$ improves the performance. We keep iterating until changing the learning rate does not improve performance anymore.

\paragraph{Using online network in VA-learning.} In theory, the VA-learning loss function 
\begin{align*}
\begin{split}
   L_\text{VA}(\theta,\phi) = \frac{1}{2}\left(V_\theta(x_t) -  \widehat{V}(x_t)\right)^2  + \frac{1}{2} \left(A_\phi(x_t,a_t) -  \widehat{A}(x_t,a_t)\right)^2,
\end{split}
\end{align*}
where the back-up targets 
\begin{align*}
\begin{split}
   \widehat{V}(x_t) &= \widehat{Q}^\pi(x_t,a_t) - A_{\phi^-}(x_{t+1},\mu),  \\
   \widehat{A}(X,A) &= \widehat{Q}^\pi(x_t,a_t) - A_{\phi^-}(x_t,\mu) - V_{\theta^-}(x_t),\   
\end{split}
\end{align*}
are computed from the target network. In deep RL implementations, we find it is important to use online network as the baseline when calculating the back-up target for the advantage function. Effectively, the back-up targets are
\begin{align*}
\begin{split}
   \widehat{V}(x_t) &= \widehat{Q}^\pi(x_t,a_t) - A_{\phi^-}(x_{t+1},\mu),  \\
   \widehat{A}(x_t,a_t) &= \widehat{Q}^\pi(x_t,a_t) - A_{\phi^-}(x_t,\mu) - V_{\theta}(x).\   
\end{split}
\end{align*}
Such a subtle change in implementation brings VA-learning and behavior dueling more similar in practice. 

\paragraph{Huber loss.}
In practice, instead of optimizing the least square loss $x^2$ function, prior work has identified that optimizing the Huber loss is a more robust alternative \citep{quandqn}
\begin{align*}
    \text{huber}(x) = x^2\mathbb{I}\left[\left|x\right|\leq \tau\right] + \left|x\right|\mathbb{I}\left[\left|x\right|> \tau\right],
\end{align*}
where by default $\tau=1$. As a result, the implemented VA-learning loss function is
\begin{align*}
\begin{split}
   L_\text{VA}(\theta,\phi) = \frac{1}{2}\text{huber}\left(V_\theta(x_t) -  \widehat{V}(x_t)\right)  + \frac{1}{2}\text{huber} \left(A_\phi(x_t,a_t) -  \widehat{A}(x_t,a_t)\right),
\end{split}
\end{align*}
while the implemented Q-learning loss function is
\begin{align*}
\begin{split}
   L_\text{QL}(\theta,\phi) = \frac{1}{2}\text{huber}\left(Q_{\theta,\phi}(x_t,a_t) -  \widehat{Q}^\pi(x_t,a_t)\right).  
\end{split}
\end{align*}
In light of this, the equivalence between VA-learning and behavior dueling no longer holds, creating a potentially bigger discrepancy in large-scale settings.

\subsection{Deep RL experiments extra results}

\begin{figure}[t]
    \centering
    \subfigure[Effect of off-policyness]{\includegraphics[keepaspectratio,width=.42\textwidth]{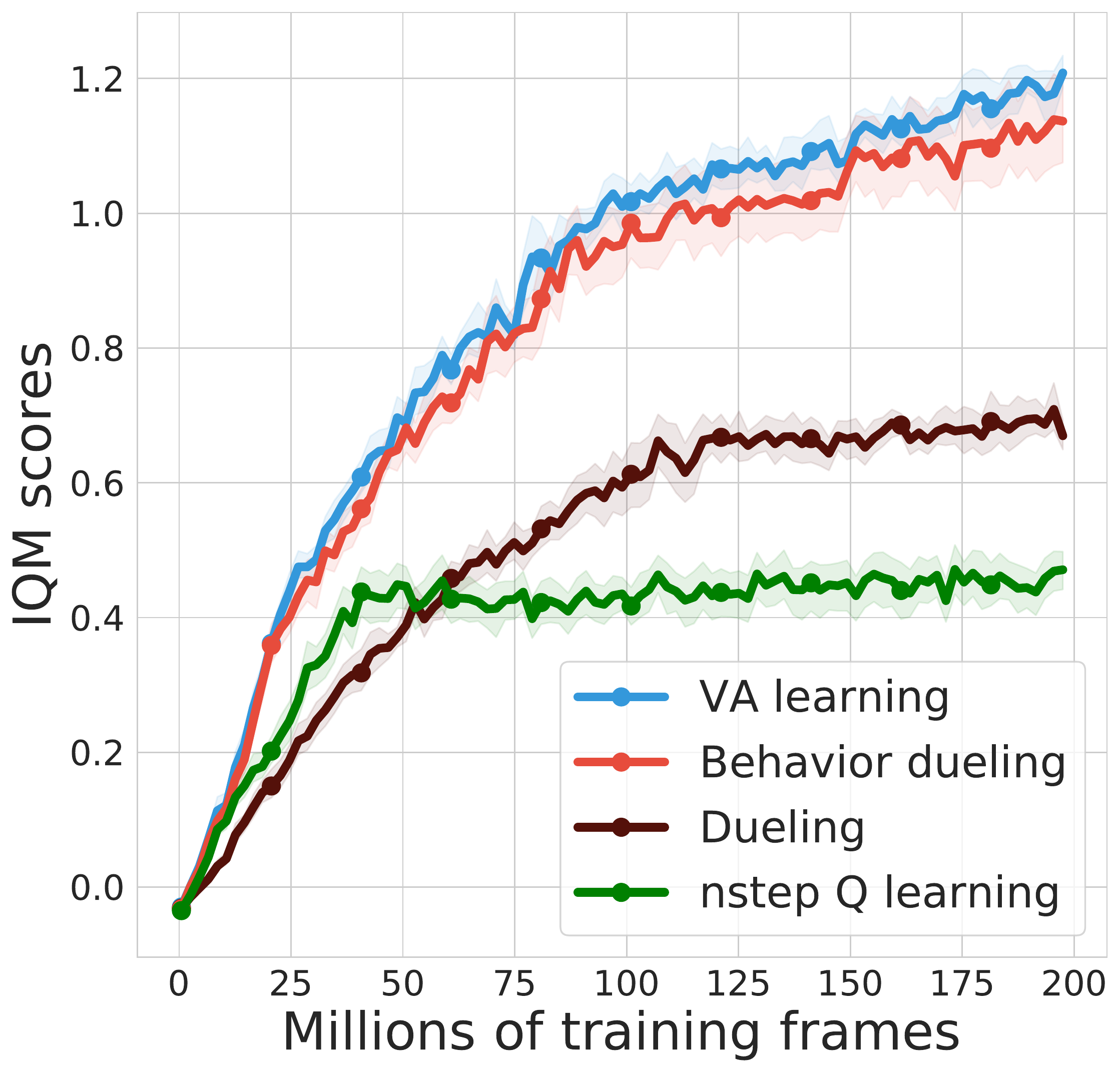}}
    \subfigure[Restricted action set ]{\includegraphics[keepaspectratio,width=.42\textwidth]{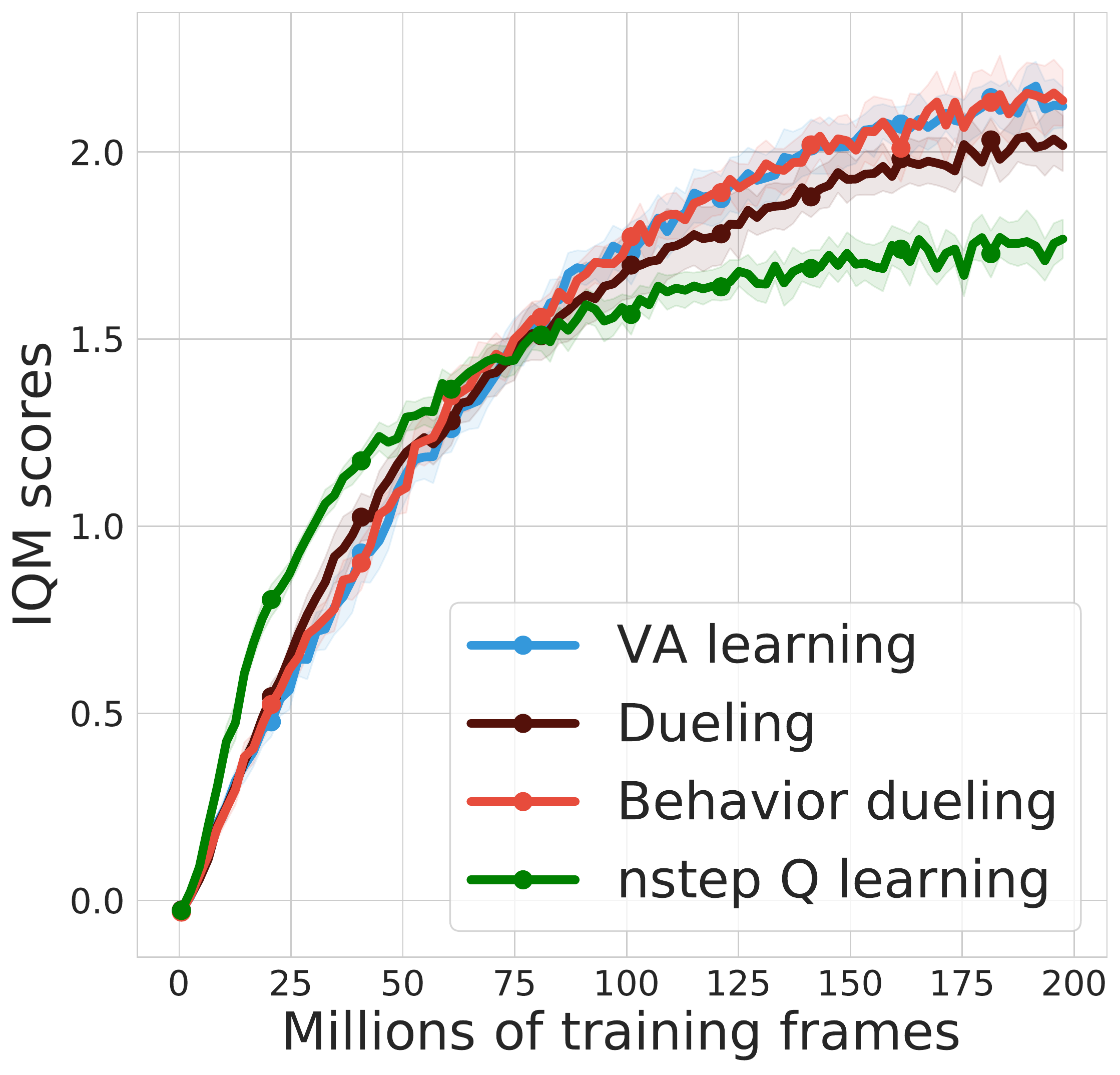}}
\caption{(a) Comparing algorithmic variants implemented with the DQN architecture in the standard Atari setup. The behavior policy is $\epsilon$-greedy to carry out exploration. The $\epsilon$ decays from $1$ to $\epsilon_f$. By default, $\epsilon_f=0.01$. Here, we set $\epsilon_f=0.5$ so that there is a large degree of off-policyness throughout training. VA-learning and behavior dueling achieves significant improvements compared to dueling and baseline Q-learning. (b) Comparing algorithmic variants implemented with the DQN architecture. The baseline agent is $n$-step Q-learning. We further compare with the dueling architecture \citep{wang2016}, the behavior dueling and VA-learning. All agents are evaluated on Atari 57 games and IQM scores \citep{agarwal2021deep} are shown across $3$ seeds. Behavior dueling and VA-learning obtain marginal advantage over dueling.}
    \label{fig:epsilon-dqn}
\end{figure}

\paragraph{Robustness to off-policyness.} We assess the robustness of various algorithmic variants to the level of off-policyness present in the replay. In DQN agents, the behavior policy $\mu$ is $\epsilon$-greedy with the rate of exploration $\epsilon$ decays from $1$ to $\epsilon_f$ over training. By default $\epsilon_f=0.01$. To increase the level of off-policyness overall in training, we set $\epsilon_f=0.5$. In Figure~\ref{fig:epsilon-dqn}(a), we see that VA-learning and behavior dueling both achieve significant performance gains over dueling, whereas the latter improves upon baseline Q-learning. This shows that VA-learning and behavior dueling are more robust to changes in data distribution which deviates from the standard setting, and is hence more robust in general.

\paragraph{Results for restricted action set.}
In Figure~\ref{fig:epsilon-dqn}(b), we compare the performance of different DQN agent variants in the standard Atari game setup. Compatible with observations made in prior work \citep{wang2016}, the dueling architecture achieves significant improvement over the $n$-step Q-learning baseline. Although $n$-step Q-learning learns faster initially, other algorithmic variants catch up as the training progresses and obtains higher asymptotic performance. 

VA-learning and behavior dueling achieve additional, albeit marginal, performance improvements over the dueling architecture. This is a sign that explicitly learning the behavior policy, rather than implicitly assuming it to be uniform, is potentially valuable. We carry out an ablation study that shows how VA-learning and behavior dueling are more robust than dueling and baseline Q-learning in a number of deep RL setups.

\paragraph{Per-game results.} In Table 1, we show the per-game result for the full action Atari game setting. Compatible with Figure~\ref{fig:allaction-dqn}, the improvement of VA-learning and behavior dueling over dueling and $n$-step Q-learning is statistically significant in most cases.

\begin{table*}[t!]
    \centering
    \caption{Per-game result in the full action set setting. We report the mean $\pm$ standard error of scores averaged across the last $5M$ frames. For each game, we highlight the method with statistically highest mean scores (multiple methods are highlighted if their confidence internval overlap). Though VA-learning, behavior dueling and uniform dueling do not improve over $n$-step Q-learning in every game, the improvement is statistically significant in most cases. This is also compatible with the aggregate results shown in Figure~\ref{fig:allaction-dqn}.
    \newline}
    \begin{sc}
    \small
    \begin{tabular}{l|c|c|c|c}\toprule[1.5pt]
        \bf Game & \bf VA-learning & \bf Behavior dueling & \bf Dueling & \bf $n$-step Q-learning \\\midrule
        alien & $0.82 \pm 0.10$ & $0.50 \pm 0.09$ & $\mathbf{1.00 \pm 0.04}$ & $1.47 \pm 0.14$ \\
amidar & $\mathbf{1.31 \pm 0.03}$ & $1.03 \pm 0.11$ & $\mathbf{1.22 \pm 0.09}$ & $0.93 \pm 0.04$ \\
assault & $\mathbf{5.61 \pm 0.28}$ & $\mathbf{5.51 \pm 0.19}$ & $4.44 \pm 0.08$ & $4.50 \pm 0.12$ \\
asterix & $2.11 \pm 0.14$ & $\mathbf{2.42 \pm 0.12}$ & $1.52 \pm 0.08$ & $1.92 \pm 0.05$ \\
asteroids & $\mathbf{0.12 \pm 0.01}$ & $\mathbf{0.10 \pm 0.01}$ & $0.04 \pm 0.00$ & $0.03 \pm 0.00$ \\
atlantis & $51.38 \pm 0.37$ & $\mathbf{53.78 \pm 0.93}$ & $52.92 \pm 1.29$ & $49.79 \pm 0.63$ \\
bank heist & $\mathbf{1.66 \pm 0.10}$ & $1.43 \pm 0.03$ & $1.44 \pm 0.03$ & $1.24 \pm 0.04$ \\
battle zone & $\mathbf{1.36 \pm 0.07}$ & $1.15 \pm 0.06$ & $1.10 \pm 0.02$ & $1.18 \pm 0.02$ \\
beam rider & $\mathbf{0.95 \pm 0.04}$ & $0.85 \pm 0.03$ & $0.78 \pm 0.03$ & $0.81 \pm 0.02$ \\
berzerk & $0.54 \pm 0.11$ & $0.50 \pm 0.10$ & $\mathbf{0.86 \pm 0.06}$ & $0.54 \pm 0.01$ \\
bowling & $0.18 \pm 0.06$ & $\mathbf{0.29 \pm 0.01}$ & $0.10 \pm 0.03$ & $0.16 \pm 0.06$ \\
boxing & $8.20 \pm 0.01$ & $\mathbf{8.24 \pm 0.01}$ & $8.20 \pm 0.01$ & $8.11 \pm 0.01$ \\
breakout & $11.73 \pm 0.16$ & $10.94 \pm 0.28$ & $\mathbf{12.41 \pm 0.10}$ & $\mathbf{12.53 \pm 0.35}$ \\
centipede & $0.19 \pm 0.00$ & $\mathbf{0.22 \pm 0.01}$ & $0.12 \pm 0.01$ & $0.07 \pm 0.02$ \\
chopper command & $\mathbf{1.39 \pm 0.02}$ & $\mathbf{1.43 \pm 0.08}$ & $1.05 \pm 0.03$ & $0.84 \pm 0.02$ \\
crazy climber & $4.96 \pm 0.06$ & $4.77 \pm 0.09$ & $4.98 \pm 0.06$ & $\mathbf{5.27 \pm 0.06}$ \\
defender & $\mathbf{3.52 \pm 0.10}$ & $\mathbf{3.56 \pm 0.11}$ & $2.85 \pm 0.05$ & $1.90 \pm 0.08$ \\
demon attack & $\mathbf{47.18 \pm 2.93}$ & $\mathbf{51.03 \pm 4.26}$ & $6.37 \pm 0.21$ & $24.23 \pm 1.97$ \\
double dunk & $\mathbf{18.60 \pm 0.10}$ & $\mathbf{18.50 \pm 0.15}$ & $17.96 \pm 0.09$ & $17.41 \pm 0.30$ \\
enduro & $1.87 \pm 0.06$ & $1.90 \pm 0.07$ & $\mathbf{2.22 \pm 0.05}$ & $1.48 \pm 0.06$ \\
fishing derby & $2.64 \pm 0.05$ & $\mathbf{2.78 \pm 0.01}$ & $2.73 \pm 0.03$ & $2.62 \pm 0.01$ \\
freeway & $1.10 \pm 0.00$ & $1.10 \pm 0.00$ & $1.11 \pm 0.00$ & $\mathbf{1.13 \pm 0.00}$ \\
frostbite & $0.49 \pm 0.22$ & $\mathbf{1.08 \pm 0.09}$ & $\mathbf{1.03 \pm 0.07}$ & $0.43 \pm 0.19$ \\
gopher & $6.02 \pm 0.22$ & $6.13 \pm 0.21$ & $5.89 \pm 0.30$ & $\mathbf{9.76 \pm 0.53}$ \\
gravitar & $\mathbf{0.31 \pm 0.05}$ & $0.17 \pm 0.00$ & $0.14 \pm 0.02$ & $0.19 \pm 0.03$ \\
hero & $\mathbf{1.24 \pm 0.01}$ & $1.20 \pm 0.00$ & $0.58 \pm 0.08$ & $0.48 \pm 0.03$ \\
ice hockey & $1.35 \pm 0.03$ & $\mathbf{1.42 \pm 0.02}$ & $0.92 \pm 0.04$ & $1.14 \pm 0.04$ \\
jamesbond & $\mathbf{35.55 \pm 7.21}$ & $25.31 \pm 6.93$ & $11.81 \pm 1.41$ & $16.96 \pm 1.17$ \\
kangaroo & $4.27 \pm 0.12$ & $\mathbf{4.43 \pm 0.06}$ & $\mathbf{4.49 \pm 0.03}$ & $3.77 \pm 0.07$ \\
krull & $7.88 \pm 0.08$ & $\mathbf{8.61 \pm 0.23}$ & $8.30 \pm 0.06$ & $7.70 \pm 0.20$ \\
montezuma revenge & $\mathbf{0.00 \pm 0.00}$ & $\mathbf{0.00 \pm 0.00}$ & $\mathbf{0.00 \pm 0.00}$ & $\mathbf{0.00 \pm 0.00}$ \\
ms pacman & $\mathbf{0.64 \pm 0.01}$ & $0.53 \pm 0.01$ & $0.58 \pm 0.04$ & $\mathbf{0.67 \pm 0.02}$ \\
name this game & $\mathbf{1.81 \pm 0.04}$ & $\mathbf{1.81 \pm 0.04}$ & $1.72 \pm 0.01$ & $1.32 \pm 0.02$ \\
phoenix & $\mathbf{10.00 \pm 0.42}$ & $\mathbf{10.39 \pm 0.73}$ & $4.02 \pm 0.32$ & $2.82 \pm 0.15$ \\
pitfall & $\mathbf{0.03 \pm 0.00}$ & $\mathbf{0.03 \pm 0.00}$ & $\mathbf{0.03 \pm 0.00}$ & $\mathbf{0.03 \pm 0.00}$ \\
pong & $1.16 \pm 0.00$ & $1.15 \pm 0.00$ & $\mathbf{1.17 \pm 0.00}$ & $\mathbf{1.17 \pm 0.00}$ \\
private eye & $\mathbf{0.00 \pm 0.00}$ & $\mathbf{0.00 \pm 0.00}$ & $\mathbf{0.00 \pm 0.00}$ & $\mathbf{0.00 \pm 0.00}$ \\
qbert & $1.51 \pm 0.04$ & $1.42 \pm 0.04$ & $1.46 \pm 0.04$ & $\mathbf{1.59 \pm 0.06}$ \\
riverraid & $1.12 \pm 0.03$ & $1.17 \pm 0.05$ & $\mathbf{1.20 \pm 0.05}$ & $1.18 \pm 0.02$ \\
road runner & $\mathbf{7.81 \pm 0.07}$ & $\mathbf{7.87 \pm 0.05}$ & $7.38 \pm 0.03$ & $7.34 \pm 0.21$ \\
robotank & $5.97 \pm 0.13$ & $6.02 \pm 0.04$ & $4.02 \pm 0.24$ & $\mathbf{6.61 \pm 0.13}$ \\
seaquest & $0.05 \pm 0.00$ & $0.03 \pm 0.00$ & $0.04 \pm 0.00$ & $\mathbf{0.43 \pm 0.04}$ \\
skiing & $\mathbf{0.73 \pm 0.02}$ & $0.66 \pm 0.03$ & $0.69 \pm 0.00$ & $-0.51 \pm 0.11$ \\
solaris & $0.01 \pm 0.00$ & $-0.01 \pm 0.01$ & $0.07 \pm 0.01$ & $\mathbf{0.09 \pm 0.02}$ \\
space invaders & $\mathbf{2.52 \pm 0.06}$ & $1.64 \pm 0.02$ & $\mathbf{2.68 \pm 0.11}$ & $\mathbf{2.84 \pm 0.37}$ \\
star gunner & $8.88 \pm 0.71$ & $\mathbf{10.63 \pm 0.27}$ & $7.53 \pm 0.10$ & $7.15 \pm 0.14$ \\
surround & $0.33 \pm 0.05$ & $0.65 \pm 0.10$ & $\mathbf{0.72 \pm 0.04}$ & $0.43 \pm 0.03$ \\
tennis & $1.52 \pm 0.01$ & $\mathbf{1.53 \pm 0.00}$ & $1.43 \pm 0.05$ & $1.19 \pm 0.11$ \\
time pilot & $\mathbf{11.65 \pm 0.36}$ & $9.07 \pm 0.51$ & $6.42 \pm 0.22$ & $6.00 \pm 0.18$ \\
tutankham & $\mathbf{1.43 \pm 0.07}$ & $1.25 \pm 0.06$ & $1.33 \pm 0.05$ & $0.50 \pm 0.05$ \\
up n down & $6.40 \pm 0.36$ & $\mathbf{7.62 \pm 0.28}$ & $6.32 \pm 0.32$ & $1.15 \pm 0.05$ \\
venture & $0.00 \pm 0.00$ & $0.06 \pm 0.03$ & $0.16 \pm 0.08$ & $\mathbf{0.79 \pm 0.05}$ \\
video pinball & $345.32 \pm 27.56$ & $194.08 \pm 19.95$ & $149.51 \pm 48.98$ & $\mathbf{417.23 \pm 26.12}$ \\
wizard of wor & $\mathbf{4.33 \pm 0.34}$ & $\mathbf{3.99 \pm 0.10}$ & $2.70 \pm 0.09$ & $0.75 \pm 0.02$ \\
yars revenge & $0.80 \pm 0.02$ & $0.57 \pm 0.21$ & $1.00 \pm 0.02$ & $\mathbf{1.11 \pm 0.02}$ \\
zaxxon & $2.92 \pm 0.06$ & $\mathbf{3.16 \pm 0.09}$ & $2.57 \pm 0.13$ & $1.43 \pm 0.05$ \\
        \bottomrule[1.46pt]
    \end{tabular}
    \end{sc}
    \label{table:full-action-per-game}
\end{table*}

\section{Connections between behavior dueling and VA-learning}
\label{appendix:why-bdueling}

We provide an in-depth discussion on the connection between behavior dueling and VA-learning in this section. Recall that in VA-learning, the value function  $V_\theta$ and advantage function $A_\phi$ are updated by minimizing the squared losses
\begin{align*}
   \frac{1}{2}\left(V_\theta(x_t) -  \widehat{V}(x_t)\right)^2 + \frac{1}{2} \left(A_\phi(x_t,a_t) -  \widehat{A}(x_t,a_t)\right)^2.
\end{align*}
as shown in Eqn~\eqref{eq:loss}. In behavior dueling, the Q-function is parameterized as $Q_{\theta,\phi}(x,a)=V_\theta(x)+A_\phi(x,a)-A_\phi(x,\mu)$. The parameters are jointly updated with gradient descents on the loss function
\begin{align}
    L_\text{QL}(\theta,\phi) = \frac{1}{2} \left(Q_{\theta,\phi}(x_t,a_t) - \widehat{Q}^\pi(x_t,a_t)\right)^2
\end{align}
as in Eqn~\eqref{eq:bdueling-q}. The gradients with respect to $\theta$ and $\phi$ correspond to the updates for the value function and advantage function components of the Q-function. We examine the value gradient and advantage gradient in detail below. Our main findings are:
\begin{itemize}
    \item Value gradients are equal in expectation for both VA-learning and behavior dueling, i.e.,
    \begin{align*}
    \mathbb{E}_\mu\left[\nabla_{\theta} L_\text{QL}(\theta,\phi)\;\middle|\;x_t\right]=\mathbb{E}_\mu\left[\nabla_{\theta} L_\text{VA}(\theta,\phi)\;\middle|\;x_t\right].
\end{align*}
\item There are subtle differences between advantage gradients differ for VA-learning and behavior dueling. Both updates bear the interpretation of fitting certain advantage function components.
\end{itemize}

\subsection{Value gradient}

We show that the value gradient of VA-learning and TD-learning with behavior dueling are equal in expectation. Similar conclusions apply for Q-learning.
\begin{restatable}{lemma}{lemmaequivalence}\label{lemma:equivalence} When the target network is the same as the online network $\theta^-=\theta,\phi^-=\phi$, then in expectation, the VA-learning value gradient is the same as the gradient of TD-learning with behavior dueling 
\begin{align*}
    \mathbb{E}_\mu\left[\nabla_{\theta} L_\text{QL}(\theta,\phi)\;\middle|\;x_t\right]=\mathbb{E}_\mu\left[\nabla_{\theta} L_\text{VA}(\theta,\phi)\;\middle|\;x_t\right],
\end{align*}
where the expectation is over the action $a_t\sim \mu(\cdot|x_t)$.
\end{restatable}
\begin{proof}
For simplicity, all our derivations below assume $\theta^-=\theta,\phi^-=\phi$.
We can write for behavior dueling
\begin{align*}
    \nabla_\theta L_\text{QL}(\theta,\phi) = \left(V_\theta(x_t) + A_\phi(x_t,a_t) - \widehat{Q}^\pi(x_t,a_t)\right) \nabla_\theta V_\theta(x_t).
\end{align*}
Now, taking expectation over the actions $a_t\sim \mu(\cdot|x_t)$ and note that $A_\phi(x_t,\mu)=0$ due to the behavior dueling parameterization, we have
\begin{align*}
    \mathbb{E}\left[L_\text{QL}(\theta,\phi)\;\middle|\;x_t\right] = \mathbb{E}\left[ \left(V_\theta(x_t) - \widehat{Q}^\pi(x_t,a_t)\right) \nabla_\theta V_\theta(x_t)\;\middle|\;x_t \right],
\end{align*}
where $\widehat{Q}^\pi(x_t,a_t) = r_t+\gamma V(x_{t+1}) +\gamma A(x_{t+1},\pi) - \gamma A(x_{t+1},\mu)$.
Examining the value gradient for the VA-learning case, we have
\begin{align*}
    \nabla_\theta L_\text{VA}(\theta,\phi) = \left(V_\theta(x_t) - \widehat{V}^\pi(x_t)\right)\nabla_\theta V_\theta(x_t).
\end{align*}
But note that $\widehat{V}^\pi(x_t) = r_t+\gamma V(x_{t+1}) +\gamma A(x_{t+1},\pi) - \gamma A(x_{t+1},\mu)$ by definition. This means
\begin{align*}
    \mathbb{E}\left[\nabla_\theta L_\text{VA}(\theta,\phi)\;\middle|\;x_t\right] =  \mathbb{E}\left[\nabla_\theta L_\text{QL}(\theta,\phi)\;\middle|\;x_t\right]
\end{align*}
and hence the proof is concluded.
\end{proof}
The above equivalence implies that both VA-learning and behavior dueling carry out value updates that fit the value function targets $\widehat{V}^\pi(x_t)$.

\subsection{Advantage gradient}
For ease of presentation, we assume a tabular parameterization for the advantage function. For VA-learning, we consider the gradient $\nabla_{A(x,a)}$ for a fixed state-action pair $(x,a)$. We can derive
\begin{align*}
    \mathbb{E}\left[\nabla_{A(x,a)} L_\text{VA}(\theta,\phi)\;\middle|\;x_t=x\right] &= \mu(a|x) \left(V(x) - \mathcal{T}^\pi Q(x,a) \right)
\end{align*}
To obtain a better intuition for the above update, note that since $V(x)$ is meant to fit the average back-up target $\mathcal{T}^\pi Q(x,a)$, the difference $V(x) - \mathcal{T}^\pi Q(x,a)$ can be understood as the residual learning target. The multiplier $\mu(a|x)$ represents the magnitude of the update, thanks to the sampling behavior distribution.

On the other hand, for behavior dueling with Q-learning, recall that we use the parameterization $A(x,a)=f(x,a) - f(x,\mu)$ and we start by considering the gradient of $\nabla_{f(x,a)}$
\begin{align*}
    \mathbb{E}\left[\nabla_{f(x,a)} L_\text{QL}(\theta,\phi)\;\middle|\;x_t=x\right] &= \mu(a|x) \left[\left(Q(x,a) - \mathcal{T}^\pi Q(x,a)\right) - \sum_b \mu(b|x) \left(Q(x,b) - \mathcal{T}^\pi Q(x,b)\right)\right]
\end{align*}
As before, the multiplier $\mu(a|x)$ is a result of the sampling distribution. The gradient can be understood as the difference between the TD error $\delta(x,a) = Q(x,a) - \mathcal{T}^\pi Q(x,a)$ and the average TD error $\sum_b \mu(b|x)\delta(x,a)$. Therefore, we can also understand the gradient to $f$ as the residual learning target. In general, however, the advantage gradient update for VA-learning and behavior dueling differ.

\subsection{Why behavior dueling is better than dueling}
Following the discussion in the main paper, we consider a parameterization $A_\nu(x,a)=f(x,a) - f(x,\nu)$ with distribution $\nu$. The notation $A_\nu$ is meant to emphasize that the advantage function depends on the distribution $\nu$. Since $A(x,a)$ represents the unshared commonents of different Q-functions, we might be interested in minimizing such unshared information. Consider the squared norm as such an objective to minimize
\begin{align*}
    \min_\nu \sum_a \mu(a|x) A_\nu(x,a)^2.
\end{align*}
\begin{restatable}{lemma}{lemmaminimizer}\label{lemma:minimizer} Across all possible parameterizations with function $f$, the unique minimizer to the weighted squared norm is $\nu=\mu$. Formally,
\begin{align*}
   \mu = \arg\min_\nu \max_f \sum_a \mu(a|x) A_\nu(x,a)^2
\end{align*}
\end{restatable}
\begin{proof}
We can rewrite the squared norm objective as
\begin{align*}
   \sum_a \mu(a|x) A_\nu(x,a)^2 = \mathbb{E}_{a\sim \mu(\cdot|x)}\left[\left(f(x,a) - f(x,\nu)\right)^2\right] \geq_{(a)} \mathbb{V}_{a\sim \mu(\cdot|x)}\left[f(x,a)\right].
\end{align*}
The equality at (a) is achieved when $f(x,\nu)=\mathbb{E}_{a\sim \mu(\cdot|x)}\left[f(x,a)\right]=f(x,\mu)$. This for any fixed $f$, the minimizing distribution $\nu$ is such that $f(x,\nu)=f(x,\mu)$. For a specific $f$, the minimizing distribution $\nu$ might not be unique. However, it is straightforward to see that across all possible distributions $\nu=\mu$ is the unique minimizer.
\end{proof}
Since the behavior dueling at $\nu=\mu$ minimizes the weighted squared norm of the advantage function, it can be understood as minimizing the unshared information across actions and hence maximizing the shared components. 

On tabular experiments, we validate such a theoretical insight in Figure~\ref{fig:adv-norm}. Across $20$ randomly generated MDPs, we run Q-learning with behavior dueling vs. dueling. At iteration $t$, let $A(x,a)=f(x,a)-f(x,\nu)$ be the advantage function of the dueling algorithms ($\nu=u$ where $u$ is uniform for dueling; and $\nu=\mu$ for behavior dueling). We measure three quantities over time: (1) for dueling, we compute $\left(f(x,a) - f(x,\nu)\right)^2$ (red); (2) for dueling, we also compute $\left(f(x,a) - f(x,\mu)\right)^2$ (blue) and (3) for behavior dueling, we compute $\left(f(x,a) - f(x,\mu)\right)^2$ (brown). All statistics are averaged over training samples, generated under the behavior policy. Comparing (2) and (3), we empirically verify that indeed, the behavior dueling parameterization obtains lower weighted advantage norm compared to the uniform dueling.
\begin{figure}[h]
    \centering
    \includegraphics[keepaspectratio,width=.42\textwidth]{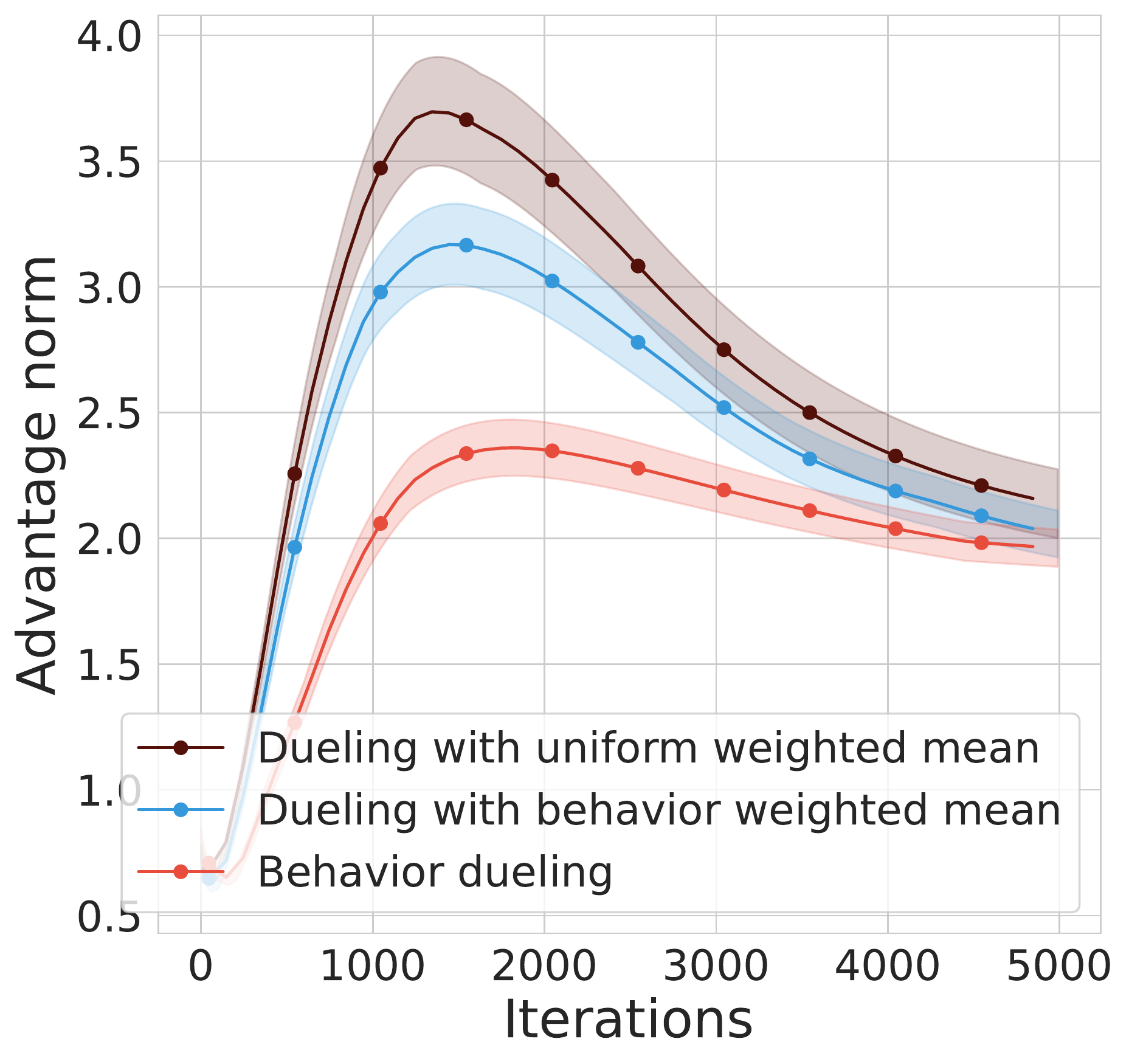}
    \caption{We compare the squared advantage norm over training iterations, across $20$ randomly generated MDPs, between behavior dueling and uniform dueling. the behavior dueling parameterization indeed obtains a lower squared norm for the advantage function compared to uniform dueling, as suggested by the theoretical arguments above.}
    \label{fig:adv-norm}
\end{figure}

\end{appendix}

\end{document}